\documentclass[nohyperref]{article}

\usepackage{microtype}
\usepackage{graphicx}
\usepackage{subfigure}
\usepackage{booktabs} 
\usepackage{caption}
\usepackage{wrapfig}
\usepackage{hyperref}
 
\usepackage[accepted]{icml2023}

\usepackage{amsmath}
\usepackage{amssymb}
\usepackage{mathtools}
\usepackage{amsthm}
\usepackage{xcolor}
\usepackage{colortbl}
\usepackage{multirow}

\usepackage[capitalize,noabbrev]{cleveref}

\theoremstyle{plain}
\newtheorem{theorem}{Theorem}[section]
\newtheorem{proposition}[theorem]{Proposition}
\newtheorem{lemma}[theorem]{Lemma}

\theoremstyle{definition}

\theoremstyle{remark}

\usepackage[normalem]{ulem}

\newcommand\revision[1]{#1} 

\newcommand{\ourmethod}{DWSNets}

\usepackage[textsize=tiny]{todonotes}

\icmltitlerunning{Equivariant Architectures for Learning in Deep Weight Spaces }

\begin{document}

\twocolumn[
\icmltitle{
Equivariant Architectures for Learning in Deep Weight Spaces }



\icmlsetsymbol{equal}{*}

\begin{icmlauthorlist}
\icmlauthor{Aviv Navon}{equal,biu}
\icmlauthor{Aviv Shamsian}{equal,biu}
\icmlauthor{Idan Achituve}{biu}
\icmlauthor{Ethan Fetaya}{biu}
\icmlauthor{Gal Chechik}{biu,nvidia}
\icmlauthor{Haggai Maron}{nvidia}
\end{icmlauthorlist}

\icmlaffiliation{biu}{Bar-Ilan University, Ramat Gan, Israel}
\icmlaffiliation{nvidia}{Nvidia, Tel-Aviv, Israel}

\icmlcorrespondingauthor{Aviv Shamsian}{aviv.shamsian@live.biu.ac.il}
\icmlcorrespondingauthor{Aviv Navon}{aviv.navon@biu.ac.il}

\icmlkeywords{Machine Learning, ICML}

\vskip 0.3in
]



\printAffiliationsAndNotice{\icmlEqualContribution} 

\begin{abstract}

Designing machine learning architectures for processing neural networks in their raw weight matrix form is a newly introduced research direction. Unfortunately, the unique symmetry structure of deep weight spaces makes this design very challenging. If successful, such architectures would be capable of performing a wide range of intriguing tasks, from adapting a pre-trained network to a new domain
to editing objects represented as functions (INRs or NeRFs). As a first step towards this goal, we present here a novel network architecture for learning in deep weight spaces. It takes as input a  concatenation of weights and biases of a pre-trained MLP and processes it using a composition of layers that are equivariant to the natural permutation symmetry of the MLP's weights: Changing the order of neurons in intermediate layers of the MLP does not affect the function it represents. We provide a full characterization of all affine equivariant and invariant layers for these symmetries and show how these layers can be implemented using three basic operations: pooling, broadcasting, and fully connected layers applied to the input in an appropriate manner. We demonstrate the effectiveness of our architecture and its advantages over natural baselines in a variety of learning tasks.

\end{abstract}


\section{Introduction}\label{sec:intro}
Deep neural networks are the primary model for learning functions from data, from classification to generation. Recently, they also became a primary model for representing data samples, for example, INRs for representing images, 3D objects, or scenes  \cite{park2019deepsdf,sitzmann2020implicit,tancik2020fourier,mildenhall2021nerf}. 
In these two cases, representing functions or data, 
it is often desirable to operate directly over the weights of a pre-trained deep model. 
For instance, given a trained deep network that performs visual object recognition, one may want to change its weights so it matches a new data distribution. 
In another example, given a dataset of INRs or NeRFs representing 3D objects, we may wish to analyze its shape space by directly applying machine learning to the raw network representation, namely the weights and biases.



In this paper, we seek a principled approach  for learning over neural weight spaces. We ask: \textit{"What neural architectures can effectively learn and process neural models that are represented as sequences of weights and biases?"}



%

The study of learning in neural weight spaces
is still in its infancy. 
Few pioneering studies \cite{eilertsen2020classifying, unterthiner2020predicting, schurholt2021self} used generic architectures such as fully connected networks and attention mechanisms to predict model accuracy or hyperparameters.
Even more recently, three papers have partially addressed the question in the context of INRs \cite{dupont2022data, xu2022signal, luigi2023deep}.  Unfortunately, it is unclear if and how these approaches could be applied to other types of neural networks since they make strong assumptions about the dimension of the input domain or the training procedure.


It remains an open problem to characterize the principles for designing deep architectures that can process the weights of other deep models. 
Traditional deep models were designed to process data 
with well-understood structures like fixed-sized tensors or sequences. In contrast, the weights of deep models live in spaces with a very different structure, which is still not fully understood \cite{hecht1990algebraic,chen1993geometry,brea2019weight,entezari2021role}. 


\textbf{Our approach.}
This paper takes a step forward toward learning in deep-weight spaces by developing architectures that account for the unique structure of these spaces in a principled manner.
More concretely, we address learning in spaces that represent a concatenation of weight (and bias) matrices of Multilayer Perceptrons (MLPs). Motivated by the recent surge of studies that incorporate symmetry into neural architectures \cite{cohen2016group,zaheer2017deep,ravanbakhsh2017equivariance,kondor2018generalization,maron2018invariant,esteves2018learning,bronstein2021geometric}, we analyze the symmetry structure of neural weight spaces and then use this analysis to design architectures that are equivariant to these symmetries. 
%
%
Specifically, we focus on the main type of symmetry found in the weights of MLPs; 
We follow a key observation, made more than 30 years ago \cite{hecht1990algebraic}, which states that for any two consecutive internal layers of an MLP, simultaneously permuting the rows of the first layer and the columns of the second layer generates a new sequence of weight matrices that represent exactly the same underlying function. To illustrate this, consider a two-layer MLP of the form $W_2\sigma(W_1x)$. Permuting the rows and columns of the weight matrices using a permutation matrix $P$ in the following way: $W_1\mapsto P^TW_1, W_2\mapsto W_2P$ will, in general, result in different weight matrices that represent \emph{exactly} the same function. More generally, any sequence of weight matrices and bias vectors can be transformed by applying permutations to their rows and columns in a similar way, while representing the same function, see Figure~\ref{fig:sym}.


After characterizing the symmetries of deep weight spaces, we define the architecture of Deep Weight-Space Networks (\emph{DWSNets}) - deep networks that process other deep networks. As with many other equivariant architectures, e.g.,  \citet{zaheer2017deep,hartford2018deep,maron2018invariant}, DWSNets are composed of simple affine equivariant layers interleaved with pointwise non-linearities. 
A key contribution of this work is that it provides the first characterization of the space of affine equivariant layers for the  symmetries of weight spaces discussed above. Interestingly, our characterization relies on the fact that the weight space is a direct sum of vector spaces corresponding to the different weights and biases in the network. Using this fact we
show that our linear equivariant layers, which we call 
\emph{DWS-layers}, have a block matrix structure where each block maps between specific weight and bias spaces of the input network. 
Furthermore, we show that these blocks can be implemented using three basic operations: broadcasting, pooling, or standard dense linear layers. This allows us to implement DWS-layers efficiently, significantly reducing the number of parameters 
compared to fully connected networks. 
%

Finally, we analyze the expressive power of DWS networks and prove that this architecture is capable of approximating a forward pass of an input network. Our findings provide a basis for further exploration of these networks and their capabilities. We demonstrate this by proving that DWS networks can approximate certain functions defined on the space of functions represented by the input MLPs. In addition, while this work focuses on MLPs, we discuss other types of input architectures, such as convolutional networks or transformers, as possible extensions.

We demonstrate the efficacy of \ourmethod{} on two types of tasks: (1) processing INRs; and (2) processing standard neural networks. The results indicate that our architecture performs significantly better than natural baselines based on data augmentation and weight-space alignment.

\textbf{Contributions.} This paper makes the following contributions: (1) It introduces a symmetry-based approach for designing neural architectures that operate in deep weight spaces; (2) It provides the first characterization of the space of affine equivariant layers between deep weight spaces; (3) It analyzes aspects of the expressive power of the proposed architecture; and (4) It demonstrates the benefits of the approach in a series of applications from INR classification to the adaptation of networks to new domains, showing advantages over natural and recent baselines.

\section{Previous Work}

In recent years several studies suggested operating directly on the parameters of NNs. In both \citet{eilertsen2020classifying, unterthiner2020predicting} the weights of trained NNs were used to predict properties of  networks. \citet{eilertsen2020classifying} suggested to predict the hyper-parameters used to train the network, and \citet{unterthiner2020predicting} proposed to predict the network generalization capabilities. Both of these studies use standard NNs on the flattened weights or on some statistics of them. 
\citet{dupont2022data} suggested applying deep learning tasks, like generative modeling, to a dataset of INRs fitted from the original data. To obtain useful representations of the data, the authors used meta-learning techniques to learn low dimensional vectors, termed modulations, which were used in normalization layers. 
Unlike this approach, our method can work on any MLP and is agnostic to the way 
it was trained. 
In \citet{schurholt2021self} the authors suggested methods to learn representations of NNs using self-supervised methods, and in \citet{schurholthyper} this approach was leveraged for NN model generation. 
\citet{xu2022signal} proposed to process neural networks by applying an NN 
to a concatenation of their high-order spatial derivatives.
\citet{peebles2022learning} proposed a generative approach to output an NN based on an initial network and a target metric such as the loss value or return. Finally, in a recent work, \citet{luigi2023deep} proposed a method 
for processing INRs using a set-like architecture \cite{zaheer2017deep}. 
See Appendix \ref{sec:more_prev} for more related work. 

\section{Preliminaries}\label{sec:prev}

\textbf{Notation.} we use $[n]=\{1,...,n\}$ and $[k,m]=\{k,k+1,\dots,m\}$. we use $\Pi_d$ for the set of $d\times d$ permutation matrices (bi-stochastic matrices with entries in $\{0,1 \}$). $S_d$ is the symmetric group of $d$ elements. $\mathbf{1}$ is an all ones vector.

\textbf{Group representations and equivariance.} Given a vector space $\mathcal{V}$ and a group $G$, a \emph{representation} is a group homomorphism $\rho$ that maps a group element $g\in G$ to an invertible matrix $\rho(g)\in GL(\mathcal{V})$. Given two vector spaces $\mathcal{V},\mathcal{W}$ and corresponding representations $\rho_1,\rho_2$ a function $L:\mathcal{V}\rightarrow \mathcal{W}$ is called \emph{equivariant} (or a $G$-linear map) if it commutes with the group action, namely $L(\rho_1(g)v)=\rho_2(g)L(v)$ for all $v\in \mathcal{V},g\in G$. When $\rho_2$ is trivial, namely the output is the same for all input transformations, $L$ is called an \emph{invariant} function. A \emph{sub-representation} of a representation $(\mathcal{V},\rho)$ is a subspace $\mathcal{W}\subseteq \mathcal{V}$ for which $\rho(g)w\in \mathcal{W}$ for all $g\in G, w\in \mathcal{W}$. A direct sum of representations $(\mathcal{W},\rho'),(\mathcal{U},\rho'')$ is a new group representation $(\mathcal{V},\rho)$ where $\mathcal{V}=\mathcal{W}\oplus\mathcal{U}$ and $\rho(g)((w,u))=(\rho'(g)w,\rho''(g)u)$. A permutation representation of a permutation group $G\leq S_n$ maps a permutation $\tau$ to its corresponding permutation matrix. An orthogonal representation maps elements of $G$ to orthogonal matrices. For an introduction to group representations, refer to \citet{fulton2013representation}.

\textbf{MultiLayer Perceptrons.}  MLPs are sequential neural networks with fully connected layers. Formally, an $M$-layer MLP $f$ is a function of the following form:
\begin{equation}
    f(x)= x_M, \quad x_{m+1}=\sigma(W_{m+1} x_m +b_{m+1}), \quad x_0=x
\end{equation}
Here, $W_m \in \mathbb{R}^{d_{m}\times d_{m-1}}$ and $b_m \in \mathbb{R}^{d_m}$, $[W_m,b_m]_{m\in[M]}$ is a concatenation of all the weight matrices and bias vectors, and $\sigma$ is a pointwise non-linearity like a ReLU  or a sigmoid. $d_m$ is the dimension of $x_m$, $m=0,\dots,M$. 

\textbf{Equivariant neural networks.} Given a group representation $(\rho,\mathcal{V})$, there are several ways to design $G$-equivariant neural networks. In this paper, we follow a popular approach \cite{zaheer2017deep,hartford2018deep,maron2018invariant} where,  in a similar fashion to Convolutional Neural Networks (CNNs), affine equivariant layers are interleaved with pointwise nonlinearities, namely, the network is of the form 
\begin{equation} \label{eq:equi_net_general}
    F_{\text{equi}}(x)= L_k\circ \sigma \circ \ldots \circ \sigma \circ L_1(x).
\end{equation}
Here, $L_i,i\in [k]$ are affine layers of the form $L(x)=Ax+b$, where $A:\mathcal{V}\to \mathcal{V}$ is a linear $G$-equivaraint function and $b:\mathcal{V}\to \mathcal{V}$ is a constant $G$-equivaraint function. For invariant tasks, we define an invariant network by composing $ F_{\text{equi}}$  with an invariant suffix: $F_{\text{inv}}(x)= h\circ L_{\text{inv}}\circ  F_{\text{equi}} $ where $L_{\text{inv}}$ is a linear invariant function and $h$ is an MLP.




\section{Permutation Symmetries of Neural Networks}\label{sec:sym}


In a pioneering work, \citet{hecht1990algebraic} observed that MLPs have permutation symmetries: swapping the order of the activations in an intermediate layer does not change the underlying function.  Motivated by previous 
works~\cite{hecht1990algebraic,brea2019weight,ainsworth2022git}, we define the \emph{weight-space} of an $M$-layer MLP as:
\begin{equation}\label{eq:direct_sum}
   \mathcal{V}=\bigoplus_{m=1}^M \left(\mathbb{R}^{d_{m} \times d_{m-1}} \oplus \mathbb{R}^{ d_m}\right):=\bigoplus_{m=1}^M \left( \mathcal{W}_m\oplus \mathcal{B}_m \right),
\end{equation}
where $\mathcal{W}_m:=\mathbb{R}^{d_{m}\times d_{m-1}}$ and $\mathcal{B}_m \vcentcolon = \mathbb{R}^{d_m}$. Each summand in the direct sum corresponds to a weight matrix and bias vector of a specific layer in the MLP, i.e., $W_m\in\mathcal{W}_m ,b_m\in \mathcal{B}_m$.  As we can independently apply any permutation to any intermediate layer of the MLP, we define the symmetry group of the weight space to be the direct product of symmetric groups for all the intermediate dimensions $m\in[1,M-1]$:
\begin{equation}\label{eq:symmetry}
    G\vcentcolon = S_{d_1}\times \cdots\times S_{d_{M-1}}.
\end{equation}

Let $v\in \mathcal{V}$, $v=[W_m,b_m]_{m\in[M]}$, then 
a group element $g=(\tau_1,\dots,\tau_{M-1})$ acts on $v$  
as follows\footnote{We note that a similar formulation first appeared in \cite{ainsworth2022git}}:
 %
\begin{subequations}\label{eq:action}
\begin{equation}
    \rho(g)v \vcentcolon = [W'_m,b'_m]_{m\in[M]}, 
\end{equation}
\begin{equation}
            W_1'=  P_{\tau_1}^T W_1, 
            ~ b_1' = P_{\tau_1}^Tb_1, 
\end{equation}
\begin{equation}
            W_m' =  P_{\tau_m}^T W_m P_{\tau_{m-1}}, 
            ~ b_m' = P_{\tau_m}^T b_m, ~m\in [2,M-1]
\end{equation}
\begin{equation}
            W_{M}' =   W_m P_{\tau_{M-1}}, 
            ~ b_{M'} = b_M. 
\end{equation}
\end{subequations}

Here, $P_{\tau_m}\in \Pi_{d_m}$ is the permutation matrix of $\tau_m\in S_{d_m}$.

\begin{figure}
    \centering
    \includegraphics[width=1.\linewidth]{ 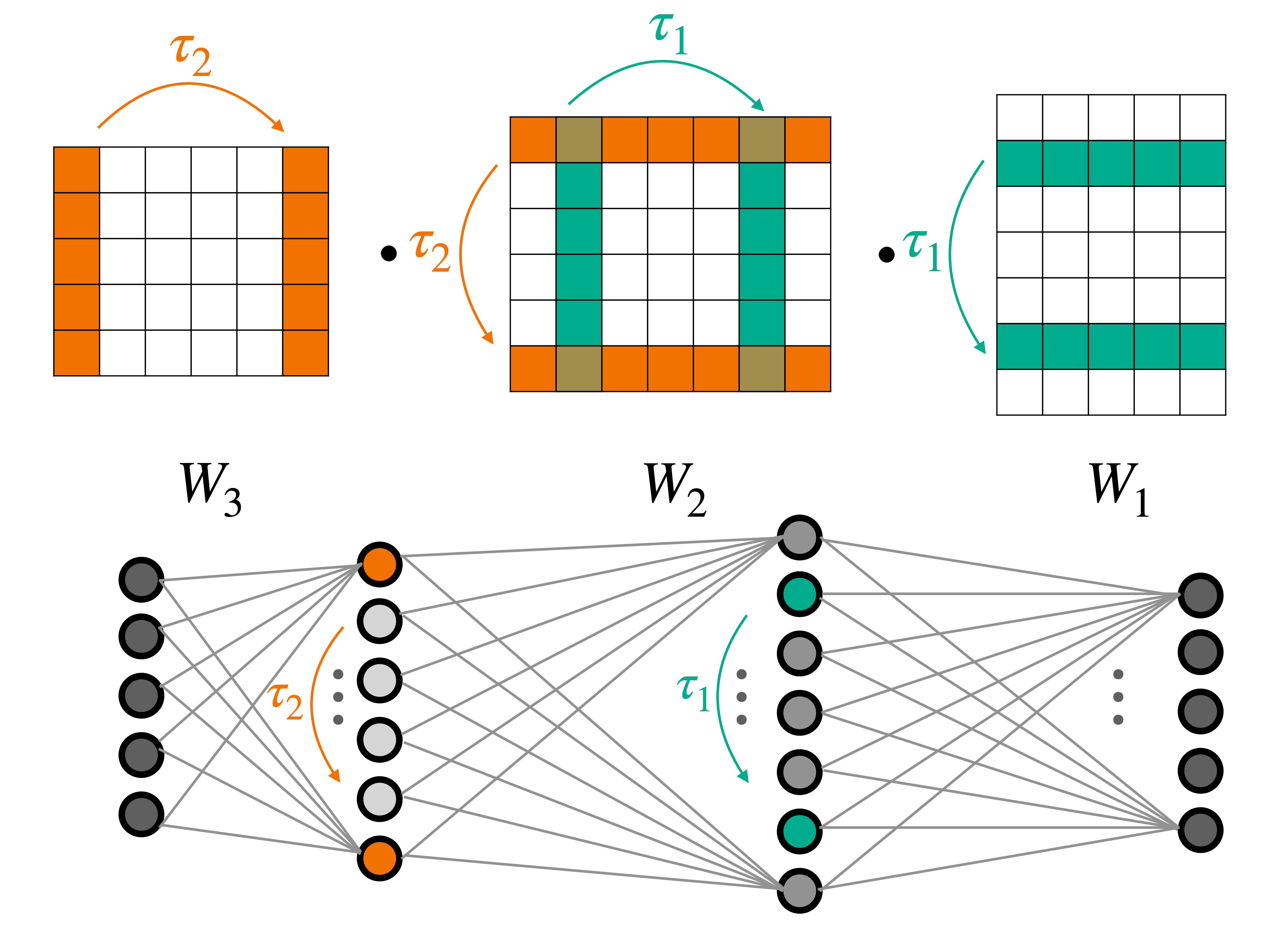}
    \caption{Symmetries of deep weight spaces, shown here on a 3-layer MLP. For any pointwise nonlinearity $\sigma$, the permutations $\tau_1,\tau_2$ can be applied to rows and columns of successive weight matrices of the MLP, without changing the function it represents.} 
    \label{fig:sym}
\end{figure}

Figure \ref{fig:sym} illustrates these symmetries for an MLP with three layers. It is straightforward to show that for any pointwise nonlinearity $\sigma$, the transformed set of parameters represents the same function as the initial set. Another simple, yet useful observation, is that all the vector spaces in  \cref{eq:direct_sum}, namely $\mathcal{W}_m,\mathcal{B}_{\ell}$, are invariant to the action we just defined, i.e., the vector space is mapped to itself under the action of $g$. This implies that $\mathcal{V}$ is a direct sum of these representations.

The symmetries described in \cref{eq:action} were used in several studies in the last few years, mainly to investigate the loss landscape of neural networks \cite{brea2019weight,tatro2020optimizing,arjevani2021analytic,entezari2021role,simsek2021geometry,ainsworth2022git,Pea2022RebasinVI}, but also in \cite{schurholt2021self} as a motivation for a data augmentation scheme. 
%
It should be noted that there are other symmetries of weight spaces that are not considered in this work~\cite{godfrey2022symmetries,bui2020functional}. One such example is scaling transformations \cite{neyshabur2015path,badrinarayanan2015understanding,bui2020functional}. 
Incorporating these symmetries into \ourmethod{} architectures is left for future work. 

\section{A Characterization of Linear Invariant and Equivariant Layers for Weight-Spaces}\label{sec:charcterize}
In this section, we describe the main building blocks of DWSNets, namely the DWS-layers. The first subsection provides an overview of the section and its main results. In the following subsections, we discuss the finer details.

\begin{figure*}[t]
    \centering
    \includegraphics[width=\linewidth]{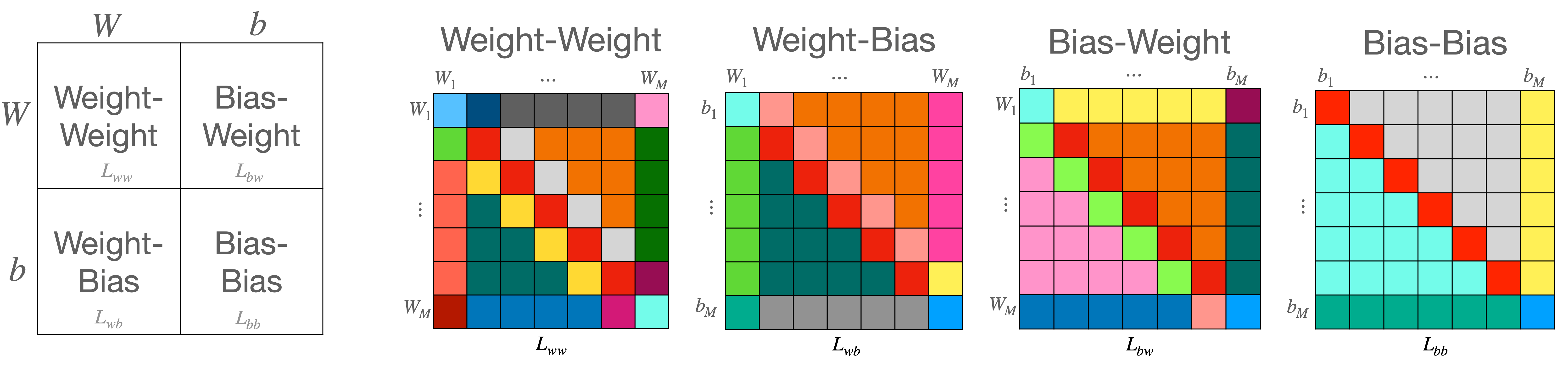}
    \caption{Block matrix structure for linear equivariant maps between weight spaces. Left: an equivariant layer for the weight space $\mathcal{V}$ to itself can be written as four blocks that map between the general weight space $\mathcal{W}$ and general bias space $\mathcal{B}$. Right: Each such block can be further written as a block matrix between specific weight and bias spaces $\mathcal{W}_m,\mathcal{B}_\ell$. Each color in each block matrix represents a different type of linear equivariant function between the sub-representations $\mathcal{W}_m,\mathcal{B}_\ell$. Blocks of the same type have different parameters. Repeating colors in different matrices are not related. See Tables \ref{tab:w2w}-\ref{tab:b2w} for a specification of the layers.}
    \label{fig:equi_layers_block}
\end{figure*}

\subsection{Overview and Main Results} \label{subsec:overview}

As explained in~\cref{eq:equi_net_general} to completely specify our architecture we need to characterize all affine equivariant and invariant maps for the weight space $\mathcal{V}$. This requires finding bases for three linear spaces: (1) the space of linear equivariant maps between the weight space $\mathcal{V}$ to itself; (2) the space of constant equivariant functions (biases) on the weight space; and (3) the space of linear invariant maps on the weight space. As we show in \cref{appx:features_and_bias},  we can readily adapt previous results for characterizing (2)-(3), and the main challenge is (1), which will be our main focus.

To find a basis for the space of equivariant layers, we will use a strategy based on a decomposition of the weight space $\mathcal{V}$ into multiple sub-representations, corresponding to the weight and bias spaces. This is based on the classic result that states that any linear equivariant map between direct sums of representations can be represented in block matrix form, with each block mapping between two constituent representations in an equivariant manner. A formal statement can be found in \cref{subsec:direct_decomp}. Importantly, this strategy simplifies our characterization and enables us to implement each block independently.

First, we introduce a coarse decomposition of $\mathcal{V}$ into two sub-representations $\mathcal{V}=\mathcal{W}\oplus \mathcal{B}$.
Here, $\mathcal{W}:=\bigoplus_{m=1}^M \mathcal{W}_m$ is a direct sum of the spaces that represent weight matrices, and $\mathcal{B}:=\bigoplus_{m=1}^M \mathcal{B}_m$ is a direct sum of spaces that represent biases. 
Based on this decomposition, we divide the layer $L$ into four linear maps that cover all equivariant linear maps between the weights $\mathcal{W}$ and the biases $\mathcal{B}$:  $L_{\text{ww}}:\mathcal{W} \rightarrow \mathcal{W}$, $L_{\text{wb}}:\mathcal{W}\rightarrow \mathcal{B}$,  $L_{\text{bw}}:\mathcal{B} \rightarrow \mathcal{W}$,   $L_{\text{bb}}:\mathcal{B} \rightarrow \mathcal{B}$.
Figure \ref{fig:equi_layers_block} (left) illustrates this decomposition. 

Our next step is constructing equivariant layers between $\mathcal{W},\mathcal{B}$, namely finding bases for the following linear spaces: $\{L_{\text{ww}}\}, \{L_{\text{wb}}\},\{L_{\text{bw}}\},\{L_{\text{bb}}\}$. This is done by splitting $\mathcal{W},\mathcal{B}$  into the sub-representations from \cref{eq:direct_sum}, i.e., $\{\mathcal{W}_m,\mathcal{B}_\ell\}_{m,\ell\in[M]}$, and characterizing all the equivariant maps between these representations. We show that all these maps are either previously characterized linear equivariant layers on sets \cite{zaheer2017deep,hartford2018deep}, or simple extensions of these layers that can be implemented using pooling, broadcast, and fully connected linear layers. This topic is discussed in detail in \cref{subsec:dws_layers}.
Figure \ref{fig:equi_layers_block} illustrates the block matrix structure of each linear map $L_{\text{ww}},L_{\text{wb}},L_{\text{bw}},L_{\text{bb}}$ according to the decomposition to sub-representations $\{\mathcal{W}_m,\mathcal{B}_\ell\}_{m,\ell\in[M]}$. Each color represents a different layer type as specified in Tables \ref{tab:w2w}-\ref{tab:b2w}.

Formally, our result can be stated as follows:

\begin{theorem} [A characterization of linear equivariant layers between weight spaces] \label{thm:char_equi}
A linear equivariant layer between the weight space $\mathcal{V}$ to itself can be written in block matrix form according to the decomposition of $\mathcal{V}$ to sub-representations $\{\mathcal{W}_m,\mathcal{B}_\ell\}_{m,\ell\in[M]}$. Moreover, each block can be implemented using a composition of pooling, broadcast, or fully connected linear layers.  Tables \ref{tab:w2w}-\ref{tab:b2w} summarize the block structure, number of parameters, and implementation of all these blocks.
\end{theorem}
As mentioned in the introduction, we call the layers from \cref{thm:char_equi} \emph{DWS-layers} and the architectures that use them (interleaved with pointwise nonlinearities), \emph{DWSNets}.

\textbf{Implementing equivariant layers.} 
The layer $L:\mathcal{V}\rightarrow \mathcal{V}$ is implemented by executing all the blocks independently and then summing the outputs according to the output sub-representations. 
To illustrate how these equivariant layers are implemented, we  write an update equation for the $m$-th weight for  $3\leq m \leq M-2$. For simplicity, we disregard input bias terms here and discuss only the weight-weight matrix \revision{presented} in Figure \ref{fig:equi_layers_block}. For an input $v\in \mathcal{V}$, $v=[W_m,b_m]_{m\in[M]}$ the update equation takes the form:

\begin{align*}
F(v)_m =& H_{\text{self}}(W_{m}) + \\
& H_{\text{adjacent}}\left(W_{m-1},W_{m+1}\right) +  \\
& H_{\text{sum}}\left( W_2,\dots,W_{m-2},W_{m+2},\dots, W_{M-1}\right) +  \\
& H_{\text{boundary}}(W_1,W_M).
\end{align*}

Here, $F(v)_m$ is the $m$-th weight matrix in $F's$ output. As can be seen, there are four different functions that are applied to the input weights according to their position in the weight sequence $W_1,\dots, W_M$ w.r.t. the output weight $m$: (1) $H_{\text{self}}$ updates the $m$-th output by applying the linear equivariant layer from \cite{hartford2018deep} \footnote{See Appendix \ref{sec:more_prev} for a full description of this layers} to the $m$-th input (red blocks in Figure \ref{fig:equi_layers_block} weight-weight panel) (2) $H_{\text{adjacent}}$ updates the $m$-th output by processing the $m-1,m+1$ weight matrices. These layers are implemented using linear equivariant DeepSets layers \cite{zaheer2017deep} 
(gray and yellow blocks in Figure \ref{fig:equi_layers_block} weight-weight panel); (3) $H_{\text{sum}}$ is a linear function that multiplies by a learnable scalar the sums of each weight matrix that is neither $m-1,m+1$ nor the first or last layer (dark green and orange blocks in Figure \ref{fig:equi_layers_block} weight-weight panel), and (4) $H_{\text{boundary}}$ are layers that are applied to the first and last weights (pink and lighter green blocks in Figure \ref{fig:equi_layers_block} weight-weight panel). These blocks are implemented by using fully connected linear layers and pooling/broadcasting operations.

We note that there are some small differences in the update rules for other $m$ values (i.e., $m\in \{1,2, M-1, M\}$). For a full description of these layers, please refer to Appendix \ref{app:tables}. 
Readers who are not interested in the technical details can continue reading in  \cref{subsec:expressive}.

\subsection{Linear Equivariant Maps for Direct Sums} \label{subsec:direct_decomp}

As mentioned above, a key property we leverage is the fact that every equivariant linear operator between direct sums of representations can be written in a block matrix form; Each block is a linear equivariant map between the corresponding sub-representations in the sum. This property is summarized in the following classical result:

\begin{proposition}[A characterisation of linear equivariant maps between direct sums of representations] \label{proposition:equi_direct_sum}
Let $(\mathcal{V}_m,\rho_m),m\in[M], \quad (\mathcal{V}'_{\ell},\rho'_{\ell}),\ell\in[M']$ be orthogonal representations of a permutation group $G$ of dimensions $d_m,d'_\ell$ respectively.  Let $(\mathcal{V},\rho) := \bigoplus_{m=1}^M\mathcal{V}_m, (\mathcal{V}',\rho') := \bigoplus_{\ell=1}^{M'}\mathcal{V}'_\ell$ be direct sums of the representations above. 
Let $B_{m\ell}$ be a basis for the space of linear equivariant functions between $(\mathcal{V}_m,\rho_m)$ to $(\mathcal{V}'_{\ell},\rho'_{\ell})$. Let  $B_{m\ell}^P$ be zero-padded versions of $B_{m\ell}$: every element of $B_{m\ell}^P$ is an all zero matrix in $\mathbb{R}^{d'\times d}$ for $d=\sum_m d_m,~d'=\sum_{\ell} d'_\ell$ except for the $(m,\ell)$ block that contains a basis element from $B_{m\ell}$. Then $B=\cup_{m\ell}B_{m\ell}^P$ is a basis for the space of linear equivariant functions from $\mathcal{V}$ to $\mathcal{V}'$.
    \end{proposition}

We refer readers to Appendix \ref{appx:proofs} for the proof. 
\begin{wrapfigure}[9]{r}{0.19\textwidth}
 \includegraphics[width=.18\textwidth]{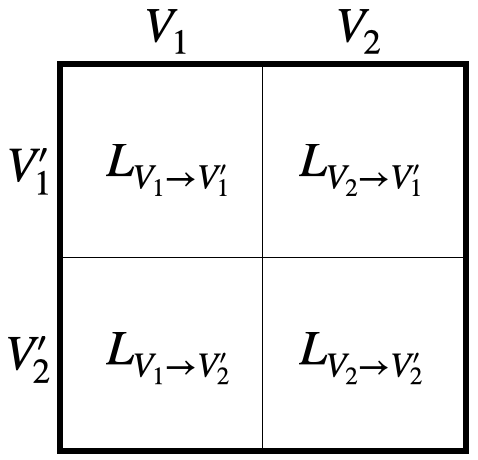}
\end{wrapfigure}
Intuitively, Proposition \ref{proposition:equi_direct_sum} reduces the problem of characterizing equivariant maps between direct sums of representations to multiple simpler problems of characterizing equivariant maps between the constituent sub-representations (see inset for an illustration). A similar observation was made in the context of irreducible representations in \citet{cohen2016steerable}. 

\subsection{Linear Equivariant Layers for Deep Weight-Spaces} \label{subsec:dws_layers}
In this subsection, we explain how to construct a basis for the space of linear equivariant functions between a weight-space to itself: $L:\mathcal{V}\rightarrow \mathcal{V}$. 

\textbf{Methodology.}  As mentioned in \cref{subsec:overview}, each linear function $L$ can be split into four maps: $L_{\text{ww}},L_{\text{wb}},L_{\text{bw}},L_{\text{bb}}$, which themselves map a direct sum of representations to another direct sum of representations. To find a basis for all such linear equivariant maps, we use Proposition \ref{proposition:equi_direct_sum} and find bases for the linear equivariant maps between all the sub-representations $\mathcal{W}_m,\mathcal{B}_{\ell}$. \revision{For simplicity, here we assume a single feature dimension and no bias terms for the DWS-layer. In Appendix \ref{appx:features_and_bias}, we discuss a simple way to extend our results to allow for multiple features and bias terms for the different DWS blocks.}

To characterize the linear equivariant layers between the subrepresentations $\{\mathcal{W}_m,\mathcal{B}_\ell\}_{m,\ell\in[M]}$, we first show how to construct layers that respect the symmetries by composing a few basic building blocks like pooling, broadcasting, and dense linear layers. We then count the number of parameters in each layer, calculated using a simple theoretical result (see Appendix \ref{appx:proofs}), to show that the layers we suggested include all linear equivariant layers.

%

\textbf{Bias-to-bias layers.}
we begin by discussing the bias-to-bias part $L_{\text{bb}}:\mathcal{B}\rightarrow\mathcal{B}$ which is the simplest case. 
%
%
%
%
%
%
%
%
%
$L_{\text{bb}}$ is composed of blocks that map between bias spaces that are of the form $T:\mathbb{R}^{d_j}\rightarrow \mathbb{R}^{d_i}$. Importantly, the indices $i,j$ determine how the map $T$ is constructed. Let us review  three examples: 
(i) When $i=j=M$, $G$  acts trivially on both spaces and the most general equivariant map 
is a fully connected linear layer. Formally, this block can be written as $b^{\text{new}}_{i}=A b^{\text{old}}_{i}$ for a parameter matrix $A\in\mathbb{R}^{d_M\times d_M}$. 
(ii) When $i=j<M$, $G$ acts jointly on the input and output by permuting them using the \emph{same} permutation. It is well known that the most general permutation equivariant layer in this case is a DeepSets layer \cite{zaheer2017deep}. Hence, this block can be written as $b^{\text{new}}_{i}=a_1 b^{\text{old}}_{i} + a_2  \mathbf{1} \mathbf{1}^T b^{\text{old}}_{i}$ for two scalar parameters $a_1,a_2\in\mathbb{R}$. 
(iii) When $i\neq j<M$ we have two dimensions on which $G$ acts by independent permutations. We show that the most general linear equivariant layer first sums on the $d_j$ dimension then multiplies the result by a learnable scalar, and finally broadcasts the result on the $d_i$ dimension. This block can be written as $b^{\text{new}}_{i}=a \mathbf{1}\mathbf{1}^T b^{\text{old}}_{j}$ for a single scalar parameter $a\in\mathbb{R}$. We refer the readers to Table~\ref{tab:b2b} for the characterization of the remaining bias-to-bias layers. The block structure of $L_{\text{bb}}$ is illustrated in the rightmost panel of \cref{fig:equi_layers_block} where the single block of type (i) is colored in blue, blocks of type (ii) are colored in red, and blocks of type (iii) are colored in gray and cyan. Note that blocks of the same type have different parameters.


%
%

\textbf{Basic operations for constructing layers between sub-representations.} In general, implementing linear equivariant maps between the sub-representations $\mathcal{W}_m,\mathcal{B}_{\ell}$  requires three basic operations: Pooling, Broadcast, and fully-connected linear maps. They will now be defined in more detail.
(1) \textit{Pooling:} A function that takes an input tensor with one or more dimensions and sums over a specific dimension. For example, for $x\in \mathbb{R}^{d_1\times d_2 }$, $POOL(d_i)$ performs summation over the $i$-th dimension; (2) \textit{Broadcast:} A function that adds a new dimension to a vector by copying information along a particular axis. $BC(d_i)$ broadcasts information along the $i$ -th dimension; (3) \textit{Linear:} A fully connected linear map that can be applied to either vectors or matrices. $LIN(d,d')$ is a linear transformation represented by a $d'\times d$ matrix. Two additional operations that can be implemented using operations (1)-(3) \footnote{See \cite{albooyeh2019incidence} for a general discussion on implementing permutation equivariant functions with these primitives.} which will be useful for us are: (i) \textit{DeepSets} \cite{zaheer2017deep}: the most general linear layer between sets; and (ii) Equivariant layers for a product of two sets as defined in \citet{hartford2018deep} (See a formal definition in Appendix \ref{sec:more_prev}). 


\textbf{Definition of layers between  $\mathcal{W}_m,~\mathcal{B}_{\ell}$.}  Let $T:\mathcal{U}\rightarrow \mathcal{U}'$  be a map between sub-representations, i.e., $ \mathcal{U}, \mathcal{U}'\in\{ \mathcal{W}_m,\mathcal{B}_\ell \}_{m,\ell\in[M]}$ 
represent a specific weight or bias space 
associated with one or two indices reflecting the layers in the input MLP they represent. For example, one such map is between $\mathcal{U}=\mathbb{R}^{d_1\times d_0}$ and $\mathcal{U}'=\mathbb{R}^{d_1}$.
We define three useful terms that will help us define a set of rules for constructing layers between these spaces; We call an index $m\in [0, M]$, a \emph{set} index (or dimension), if $G$ acts on it by permutation, otherwise, we call it \emph{free} index. From the definition, it is clear that $0,M$ are free indices while all other indices, namely $m\in[1,M-1]$ are set indices. Additionally, if indices in the domain and codomain are the same, we call them \emph{shared} indices. 

Based on the basic operations described above, the following rules are used to define equivariant layers between sub-representations $\mathcal{W}_m,\mathcal{B}_{\ell}$.
(1) In the case of two shared set indices, which happens when mapping $\mathcal{W}_m,~m\in[2,M-1]$ to itself, we use \citet{hartford2018deep}. (2) In the case of a single shared set index, for example, when mapping $\mathcal{B}_m,~m\in[1,M-1]$ to itself we use DeepSets~\cite{zaheer2017deep}. (3) When both the domain and the codomain have free indices, we use a dense linear layer. For example when mapping $\mathcal{B}_M$ to itself. 
(4) We use pooling to contract unshared set input dimensions and linear layers to contract free input dimensions and, (5) We use broadcasting to extend output set dimensions, and linear layers to extend output free dimensions. Tables \ref{tab:w2w}-\ref{tab:b2w} provide a complete specification of linear equivariant layers between all sub-representations  $\{ \mathcal{B}_m,\mathcal{W}_\ell \}_{m,\ell\in[M]}$.


\textbf{Proving that these layers form a basis.} 
At this point, we have created a list of equivariant layers between all sub-representations. \revision{We still have to prove that these transformations are linearly independent and that they span the space of linear equivariant maps between the corresponding representations. First, by using Proposition \ref{proposition:equi_direct_sum}, we only need to demonstrate that the parameters in each block are independent. Hence, 
the linear independence results can be directly obtained from previous works~\cite{zaheer2017deep,hartford2018deep}, 
or easily derived by writing the block operators as vectors. Finally, to show the proposed layers span the space of linear equivariant maps between the corresponding representations, we employ a dimension-counting argument: } 
we calculate the dimension of the space of linear equivariant maps for each pair of representations and show that the number of independent parameters in each proposed layer is equal to this dimension. 
See proof in Appendix \ref{appx:subrep_equi_maps}.

\textbf{Extension to nonlinear aggregation mechanisms.} Similarly to previous works that considered equivariance to permutations \cite{qi2017pointnet,zaheer2017deep,velivckovic2017graph,lee2019set} we can replace any summation term in our layers with either a non-linear aggregation function like $\texttt{max}$, or more complex attention mechanisms.

\textbf{Multiple channels, invariant layers and biases.} We refer the reader to Appendix \ref{appx:features_and_bias} for a characterization of the bias terms (of DWSNets), linear invariant layers, and a generalization of Theorem \ref{thm:char_equi} to multiple input and output channels.

\begin{figure}[t]
    \centering
    \includegraphics[width=.925\linewidth, clip]{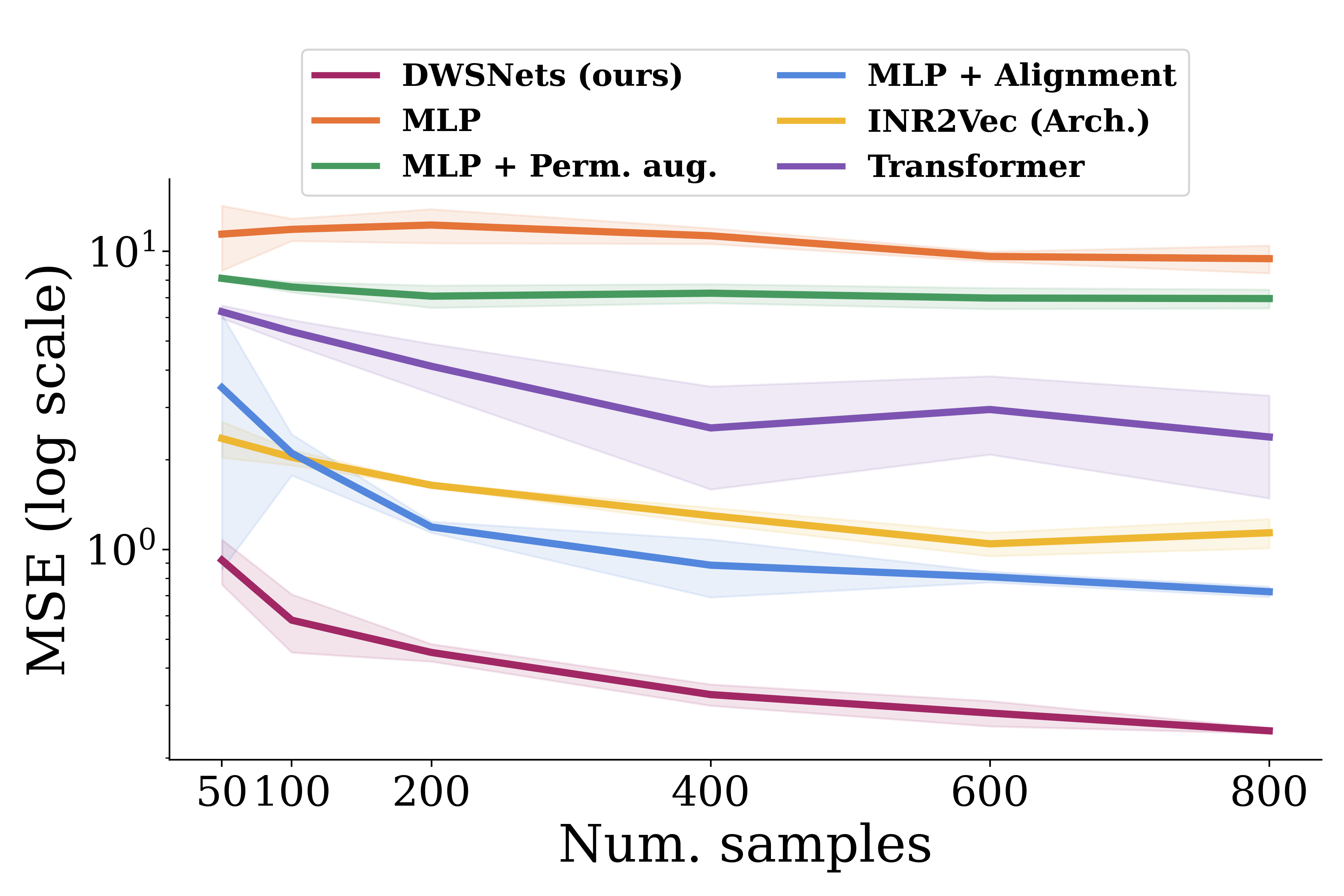}
    \caption{\textit{Sine wave regression}. Test MSE (log scale) for a varying number of training examples.}
    \label{fig:sine}
\end{figure}


\subsection{Extension to Other Input Architectures}\label{subsec:exten}



In this paper, we primarily focus on MLP architectures as the input  for \ourmethod{}. However, the characterization can be extended to additional architectures. Here we discuss the extension to two additional architectures, namely convolutional neural networks (CNNs) and Transformers \cite{vaswani2017attention}. 
Convolution layers consist of weight matrices $W_i\in\mathbb{R}^{k_{i}^1\times k_{i}^2\times d_{i-1}\times d_i}$ and biases $b_i\in\mathbb{R}^{d_i}$,
where $d_{i-1}$ and $d_i$ represents the input and output channel dimensions, respectively. As with MLPs, simultaneously permuting the channel dimensions of adjacent layers would not change the function represented by the CNN. Concretely, consider a 2-layers CNN with weights $W_1, W_2, b_1, b_2$ and let $\tau \in S_{d_1}$. Applying $\tau$ to the $d_1$ dimension of $W_1, W_2$ and $b_1$ will not change the function represented by the CNN.
Transformers consist of self-attention layers followed by feed-forward layers applied independently to each position. Let $W_i^V, W_i^K, W_i^Q \in \mathbb{R}^{d\times d'}$ denote the value, key, and query weight matrices and let $P\in \Pi_{d'}$. One symmetry in this setup can be described by setting $W_i^Q\mapsto W_i^Q P$ and $W_i^K\mapsto W_i^K P$ which would not change the function.

\section{Expressive Power} \label{subsec:expressive}
The expressive power of equivariant networks is an important concern since by restricting our hypothesis class we might unintentionally impair its function approximation capabilities. 
For example, this is the case with Graph neural networks \cite{morris2019weisfeiler, xu2018powerful, morris2021weisfeiler}. Here, we provide a first step towards understanding the expressive power of \ourmethod{} by demonstrating that these networks are capable of approximating feed-forward procedures on input networks.

\begin{table}
\centering
\caption{\textit{INR classification:} The class of an INR is defined by the image that it represents. We report the average test accuracy.}
\vskip 0.11in
\small
\begin{tabular}{lcc}
\toprule
& MNIST INR & Fashion-MNIST INR \\
\midrule
MLP & $17.55 \pm 0.01$ & $19.91\pm 0.47$ \\
MLP + Perm. aug  & $29.26 \pm 0.18$ &  $22.76\pm 0.13$ \\
MLP + Alignment & $58.98\pm 0.52$  & $47.79 \pm 1.03$    \\
INR2Vec (Arch.) & $23.69\pm 0.10$ & $22.33 \pm 0.41$  \\
Transformer & $26.57\pm 0.18$ & $26.97\pm 0.33$  \\
\midrule
DWSNets (ours) & $\mathbf{85.71\pm 0.57}$  &  $\mathbf{67.06 \pm 0.29}$ \\
\bottomrule
\end{tabular}
\label{tab:mnsit_clf}
\end{table}

\begin{proposition}[\ourmethod{} can approximate a feed-forward pass]\label{prop:expressive}
Let $M,d_0,\dots,d_M$ specify an MLP architecture with ReLU nonlinearities. Let $K \subset \mathcal{V}$, $K' \subset  \mathbb{R}^{d_0}$ be compact sets. DWSNets with ReLU nonlinearities are capable of uniformly approximating a feed-forward procedure on an input MLP  represented as a weight vector $v\in K $ and an input to the MLP,  $x \in K'$. \end{proposition}

The proof can be found in Appendix \ref{appx:expressive}.
Importantly, this inherent ability of DWSNets to evaluate input networks could be a very useful tool, for example, in order to separate MLPs that represent different functions. 
As another example, below, we show that \ourmethod{} can approximate any ``nice" function defined on the space of functions represented by MLPs with weights in some compact subset of $\mathcal{V}$. 

\begin{proposition}\label{prop:approx}
    (informal) Let $g:\mathcal{F_V}\rightarrow \mathbb{R}$ be a function defined on the space of functions represented by $M$-layer ReLU MLPs with dimensions $d_0,...,d_M$, whose weights are in a compact subset of $\mathcal{V}$ and their input domain is a compact subset of $\mathbb{R}^{d_0}$. Assume that $g$ is $L$-Lipshitz w.r.t  $||\cdot ||_\infty$ (on functions), then under some additional mild assumptions specified in Appendix \ref{appx:approx}, DWSNets with ReLU nonlinearities are capable of uniformly approximating $g$. 
\end{proposition}
The proof can be found in Appendix \ref{appx:approx}.
We note that \cref{prop:approx} differs from most universality theorems in the relevant literature \cite{maron2019universality,keriven2019universal} since we do not prove that we can approximate any $G$-equivariant function on $\mathcal{V}$. In contrast, we show that \ourmethod{} are powerful enough to approximate functions on the function space defined by the input MLPs, that is, functions that give the same result to all weights that represent the same functions.

\section{Experiments}\label{sec:Experiments}

We evaluate \ourmethod{} in two families of tasks. (1) First, taking input networks that represent data, like INRs \cite{park2019deepsdf, sitzmann2020implicit}. Specifically, we train a model to classify INRs based on the class of the image they represent or predict continuous properties of the objects they represent. (2) Second, taking input networks that represent standard input-output mappings such as image classifiers. We train a model to operate on these mappings and adapting them to new domains. We also perform additional experiments, for example predicting the generalization performance of an image classifier in Appendix~\ref{sec:additional_exp}. Full experimental and technical details are discussed in Appendix~\ref{app:exp_details}. \revision{To support future research and the reproducibility of our results, we made our source code and datasets publicly available at: \textcolor{magenta}{\url{https://github.com/AvivNavon/DWSNets}}}. 

\begin{figure}[t]
    \centering
    \includegraphics[width=1.\linewidth, clip]{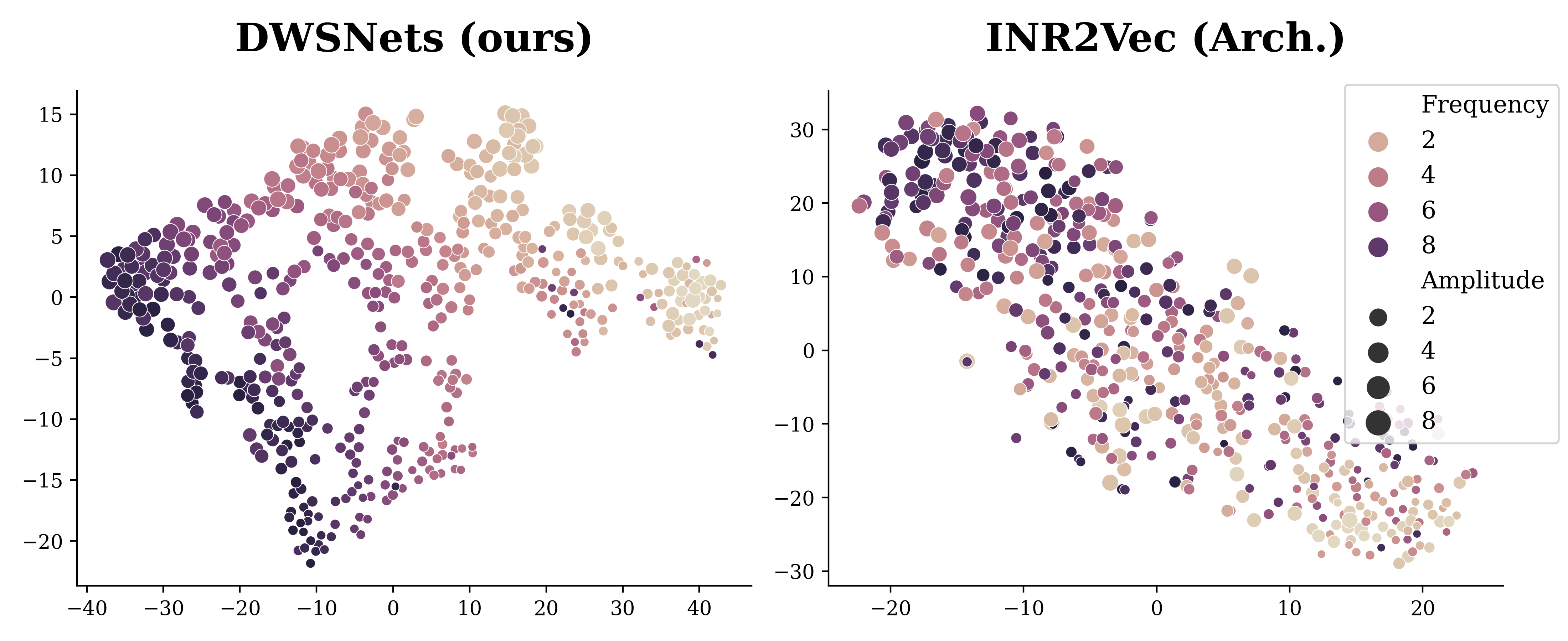}
    \caption{\textit{Dense representation:} 2D TSNE of the resulting low-dimensional space. We present the results for \ourmethod{} and the second best performing baseline, INR2Vec (architecture). See Appendix~\ref{app_sec:dense_rep} for full results.}
    \label{fig:dense_main}
\end{figure}

\textbf{Baselines.} Our objective in this section is to compare different \emph{architectures} that operate directly on weight spaces, using the same data, loss function, and training procedures. 
As learning on weight spaces is a relatively new problem, we consider five natural and recent baselines. \textbf{(i) \textit{MLP} :} A standard MLP is applied to a vectorized version of the weight space. \textbf{(ii) \textit{MLP + augmentations}:}, apply the MLP from (i) but with permutation-based data augmentations sampled randomly from the symmetry group $G$. \textbf{(iii) \textit{MLP + weight alignment}: } We perform a weight alignment procedure prior to training using the algorithm recently suggested in \cite{ainsworth2022git}, see full details in Appendix~\ref{app:exp_details}. 
\textbf{(iv) \textit{INR2Vec}:} The architecture suggested in \cite{luigi2023deep} (see Appendix \ref{sec:more_prev} for a discussion). 
Note we do not use their pre-training since we are interested in comparing only the architectures. \textbf{(v) \textit{Transformer}:} The architecture of \cite{schurholt2021self}. It adapts the transformer encoder architecture and attends between different rows in weight and bias matrices to form a global representation of the input network.



\textbf{Data preparation.} We train all input networks independently, each starting with a different random seed, in order to test our architecture on diverse data obtained from multiple independent sources. 
Unless stated otherwise, we train a \emph{single} copy for each network, e.g., a single INR per image. 

\subsection{Results}

\textbf{Regression of sine wave frequency.}
To first illustrate the operation of \ourmethod{}, we look into a regression problem. 
We train INRs to fit sine waves on $[-\pi, \pi]$, with different frequencies sampled from $U(0.5, 10)$. 
The task is to have the DWSNet predict the frequency of a given test INR network. To illustrate the generalization capabilities of the architectures, we repeat the experiment by training 
with a varying number of training examples (INRs). Figure~\ref{fig:sine} shows that \ourmethod{} performs significantly better than baseline methods even with a small number of training examples. 

\textbf{Classification of images represented as INRs.} Here, INRs were trained to represent images from MNIST~\cite{lecun1998gradient} and Fashion-MNIST~\cite{xiao2017fashion}. The task is to 
recognize the image class, like the digit in MNIST, by using the weights of 
input INR. Table~\ref{tab:mnsit_clf} shows that \ourmethod{} outperforms all baseline methods by a large margin. 





\begin{table}
\centering
\small
\vspace{-0.2in}
\caption{\textit{Dense representation of sine wave INRs:} MSE of a linear regressor that predicts \textit{frequency} and \textit{amplitude}.} 
\vskip 0.11in
\begin{tabular}{lc}
\toprule
& MSE  \\
\midrule
MLP & $7.39 \pm 0.19$ \\
MLP + Perm. aug  & $5.65 \pm 0.01$ \\
MLP + Alignment & $4.47 \pm 0.15$   \\
INR2Vec (Arch.) &  $3.86 \pm 0.32$  \\
Transformer & $ 5.11 \pm 0.12$  \\ 
\midrule
DWSNets (ours) & $\mathbf{1.39 \pm 0.06}$ \\
\bottomrule
\end{tabular}
\label{tab:dense_rep}
\end{table}

\textbf{Self-supervised learning for dense representation.} Here we wish to embed neural networks into a semantic coherent low dimensional space, similar to~\citet{schurholthyper}. To that end, we fit INRs on sine waves of the form $a\sin(bx)$ on $[-\pi, \pi]$. Here $a,b\sim U(0, 10)$ and $x$ is a grid of size $2000$. We use a SimCLR-like training procedure and objective~\cite{chen2020simple}: Following ~\citet{schurholthyper},
we generate random views from each INR by adding Gaussian noise (with a standard deviation of 0.2) and random masking (with a probability of 0.5). We evaluate the different methods in two ways. First, we qualitatively observe a 2D TSNE of the resulting space. The results are presented in Figures~\ref{fig:dense_main} and~\ref{fig:dense_all}. 
For quantitative evaluation, we train a (linear) regressor for predicting $a, b$ on top of the  embedding space obtained by each method. See results in Table~\ref{tab:dense_rep}.

\textbf{Learning to adapt networks to new domains.}
Here we train a model to adapt a classification model to a new domain. Specifically, given an input weight vector $v$, we wish to output residual weights $\Delta v$ such that a classification network parametrized using $v-\Delta v$ performs well on the new domain. It is natural to require that $\Delta v$ will be permuted if $v$ is permuted, and hence a $G$-equivariant architecture is appropriate. At test time, our model can adapt an unseen classifier to the new domain using a single forward pass.
%
%
Using the CIFAR10~\cite{krizhevsky2009learning} dataset as the source domain, we train multiple image classifiers. To increase the diversity of the input classifiers, we train each classifier on the binary classification task of distinguishing between two randomly sampled classes. For the target domain, we use a version of CIFAR10 corrupted with random rotation, flipping, Gaussian noise, and color jittering. The results are presented in Table~\ref{tab:cifar_da}. Note that in test time the model should generalize to unseen image classifiers, as well as unseen images.




\begin{table}
\centering
\caption{\textit{Adapting a network to a new domain.} Test accuracy of CIFAR-10-Corrupted models adapted from CIFAR-10 models.}
\vskip 0.11in
\small
\begin{tabular}{lc}
\toprule
& CIFAR10 $\to$ CIFAR10-Corrupted \\
\midrule
No adaptation & $60.92 \pm 0.41$ \\
\midrule
MLP & $64.33 \pm 0.36$ \\
MLP + Perm. aug  & $64.69 \pm 0.56$ \\
MLP + Alignment & $67.66 \pm 0.90$   \\
INR2Vec (Arch.) & $65.69\pm 0.41$  \\
Transformer & $61.37\pm 0.13$  \\ 
\midrule
DWSNets (ours) & $\mathbf{71.36\pm 0.38}$ \\
\bottomrule
\end{tabular}
\label{tab:cifar_da}
\end{table}

\textbf{Multiple INR views as data augmentation.} We investigate the impact of training with multiple INR views (copies) for each image on the performance of our model. 
\begin{wraptable}[13]{r}{0.25\textwidth}
\centering
\small
\vspace{-0.2in}
\caption{\textit{Fashion-MNIST multi-view INR classification}: Test results for training with a varying number of INR views per image.}
\vskip 0.11in
\begin{tabular}{cc}
\toprule
\# INRs & Acc.  \\
\midrule
1 & $67.06 \pm 0.29$ \\
2 & $70.22 \pm 0.38$ \\
4 & $70.31 \pm 0.09$ \\
6 & $73.32 \pm 0.11$   \\
8 &  $74.87 \pm 0.18$  \\
10 & $ 75.12 \pm 0.05$  \\ 
\bottomrule
\end{tabular}
\vskip 0.11in
\label{tab:numltiviews}
\end{wraptable}
We return to the Fashion-MNIST INR classification task, using a varying number of copies $k\in\{1,\dots,10\}$. The results, presented in Table~\ref{tab:numltiviews},
show that incorporating a diverse set of INRs per image through random initializations significantly improves the model's generalization capabilities (by $\sim 8\%$).
Our results highlight the importance of establishing an adequate evaluation protocol for DWS models and experiments (e.g., by using the same number of INR copies for training each model).

\subsection{Analysis of the Results} 
In this section, we evaluated \ourmethod{} on several learning tasks and showed that it outperforms all other methods, usually by a large margin. 
Also, compared to the most natural baseline of network alignment, \ourmethod{} scale significantly better with the data. In reality, it is challenging to use this baseline 
since the weight-space alignment problem is hard \cite{ainsworth2022git}. The problem is further amplified when having large input NNs or large (networks) datasets.



\section{Conclusion and Future Work}\label{sec:con}
This paper considers the problem of applying neural networks directly on neural weight spaces. We present a principled approach and propose an architecture for the network that is equivariant to a large group of natural symmetries of weight spaces. We hope this paper will be one of the first steps towards neural models capable of processing weight spaces efficiently in the future.

\textbf{Limitations.} 
One limitation of our method is that an equivalent layer structure is currently tailored to a specific MLP architecture. However, 
this can be alleviated in the future, for example by sharing the parameters of the equivariant blocks between inner layers. Also, we found it difficult to train \ourmethod{} on some learning tasks, presumably because finding a suitable weight initialization scheme for \ourmethod{} was hard. See Appendix \ref{app:failure-cases} for a discussion on these cases.  
Finally, the implementation of our \ourmethod{} is somewhat complicated. We made our code and data publicly available so that others can build on it and improve it.

\textbf{Future work.}
Several potential directions for future research could be explored, including modeling other weight space symmetries in architectures, understanding how to initialize the weights of \ourmethod{}, and studying the approximation power of \ourmethod{}. Other worthwhile directions are finding efficient data augmentation schemes for training on weight spaces, \revision{extend DWSNets to allow heterogeneous input networks,} and incorporating permutation symmetries for other types of input architectures.



\section{Acknowledgements}

\revision{The authors wish to thank Nadav Dym and Derek Lim for providing valuable feedback on early versions of the
manuscript, 
and Yaron Lipman for the helpful discussions. 
This study was funded by
a grant to GC from the Israel Science Foundation (ISF 737/2018), and by an equipment grant to GC and Bar-Ilan
University from the Israel Science Foundation (ISF 2332/18). AN and AS are supported by a grant from the Israeli
higher-council of Education, through the Bar-Ilan data science institute. IA is supported by a PhD fellowship 
from the Israeli Council for higher education.}

\bibliography{main_arxiv.bib}

\begin{thebibliography}{76}
\providecommand{\natexlab}[1]{#1}
\providecommand{\url}[1]{\texttt{#1}}
\expandafter\ifx\csname urlstyle\endcsname\relax
  \providecommand{\doi}[1]{doi: #1}\else
  \providecommand{\doi}{doi: \begingroup \urlstyle{rm}\Url}\fi

\bibitem[Agarap(2018)]{agarap2018deep}
Agarap, A.~F.
\newblock Deep learning using rectified linear units (relu).
\newblock \emph{arXiv preprint arXiv:1803.08375}, 2018.

\bibitem[Agustsson \& Timofte(2017)Agustsson and
  Timofte]{Agustsson_2017_CVPR_Workshops}
Agustsson, E. and Timofte, R.
\newblock Ntire 2017 challenge on single image super-resolution: Dataset and
  study.
\newblock In \emph{The IEEE Conference on Computer Vision and Pattern
  Recognition (CVPR) Workshops}, July 2017.

\bibitem[Ainsworth et~al.(2022)Ainsworth, Hayase, and
  Srinivasa]{ainsworth2022git}
Ainsworth, S.~K., Hayase, J., and Srinivasa, S.
\newblock Git re-basin: Merging models modulo permutation symmetries.
\newblock \emph{arXiv preprint arXiv:2209.04836}, 2022.

\bibitem[Albooyeh et~al.(2019)Albooyeh, Bertolini, and
  Ravanbakhsh]{albooyeh2019incidence}
Albooyeh, M., Bertolini, D., and Ravanbakhsh, S.
\newblock Incidence networks for geometric deep learning.
\newblock \emph{arXiv preprint arXiv:1905.11460}, 2019.

\bibitem[Arjevani \& Field(2021)Arjevani and Field]{arjevani2021analytic}
Arjevani, Y. and Field, M.
\newblock Analytic study of families of spurious minima in two-layer relu
  neural networks: a tale of symmetry ii.
\newblock \emph{Advances in Neural Information Processing Systems},
  34:\penalty0 15162--15174, 2021.

\bibitem[Ashmore \& Gashler(2015)Ashmore and Gashler]{ashmore2015method}
Ashmore, S. and Gashler, M.
\newblock A method for finding similarity between multi-layer perceptrons by
  forward bipartite alignment.
\newblock In \emph{2015 International Joint Conference on Neural Networks
  (IJCNN)}, pp.\  1--7. IEEE, 2015.

\bibitem[Azizian \& Lelarge(2021)Azizian and Lelarge]{azizian2020expressive}
Azizian, W. and Lelarge, M.
\newblock Expressive power of invariant and equivariant graph neural networks.
\newblock In \emph{9th International Conference on Learning Representations,
  {ICLR}}, 2021.

\bibitem[Badrinarayanan et~al.(2015)Badrinarayanan, Mishra, and
  Cipolla]{badrinarayanan2015understanding}
Badrinarayanan, V., Mishra, B., and Cipolla, R.
\newblock Understanding symmetries in deep networks.
\newblock \emph{arXiv preprint arXiv:1511.01029}, 2015.

\bibitem[Baker et~al.(2018)Baker, Gupta, Raskar, and Naik]{BakerGRN18}
Baker, B., Gupta, O., Raskar, R., and Naik, N.
\newblock Accelerating neural architecture search using performance prediction.
\newblock In \emph{6th International Conference on Learning Representations,
  {ICLR} 2018, Vancouver, BC, Canada, April 30 - May 3, 2018, Workshop Track
  Proceedings}, 2018.

\bibitem[Brea et~al.(2019)Brea, Simsek, Illing, and Gerstner]{brea2019weight}
Brea, J., Simsek, B., Illing, B., and Gerstner, W.
\newblock Weight-space symmetry in deep networks gives rise to permutation
  saddles, connected by equal-loss valleys across the loss landscape.
\newblock \emph{arXiv preprint arXiv:1907.02911}, 2019.

\bibitem[Bronstein et~al.(2021)Bronstein, Bruna, Cohen, and
  Veli{\v{c}}kovi{\'c}]{bronstein2021geometric}
Bronstein, M.~M., Bruna, J., Cohen, T., and Veli{\v{c}}kovi{\'c}, P.
\newblock Geometric deep learning: Grids, groups, graphs, geodesics, and
  gauges.
\newblock \emph{arXiv preprint arXiv:2104.13478}, 2021.

\bibitem[Bui Thi~Mai \& Lampert(2020)Bui Thi~Mai and
  Lampert]{bui2020functional}
Bui Thi~Mai, P. and Lampert, C.
\newblock Functional vs. parametric equivalence of relu networks.
\newblock In \emph{8th International Conference on Learning Representations},
  2020.

\bibitem[Chang et~al.(2019)Chang, Flokas, and Lipson]{chang2019principled}
Chang, O., Flokas, L., and Lipson, H.
\newblock Principled weight initialization for hypernetworks.
\newblock In \emph{International Conference on Learning Representations}, 2019.

\bibitem[Chen et~al.(1993)Chen, Lu, and Hecht-Nielsen]{chen1993geometry}
Chen, A.~M., Lu, H.-m., and Hecht-Nielsen, R.
\newblock On the geometry of feedforward neural network error surfaces.
\newblock \emph{Neural computation}, 5\penalty0 (6):\penalty0 910--927, 1993.

\bibitem[Chen et~al.(2020)Chen, Kornblith, Norouzi, and Hinton]{chen2020simple}
Chen, T., Kornblith, S., Norouzi, M., and Hinton, G.
\newblock A simple framework for contrastive learning of visual
  representations.
\newblock In \emph{International conference on machine learning}, pp.\
  1597--1607. PMLR, 2020.

\bibitem[Cohen \& Welling(2016)Cohen and Welling]{cohen2016group}
Cohen, T. and Welling, M.
\newblock Group equivariant convolutional networks.
\newblock In \emph{International conference on machine learning}, pp.\
  2990--2999. PMLR, 2016.

\bibitem[Cohen \& Welling(2017)Cohen and Welling]{cohen2016steerable}
Cohen, T.~S. and Welling, M.
\newblock Steerable cnns.
\newblock In \emph{5th International Conference on Learning Representations,
  {ICLR}}, 2017.

\bibitem[Cohen et~al.(2018)Cohen, Geiger, K{\"{o}}hler, and
  Welling]{cohen2018spherical}
Cohen, T.~S., Geiger, M., K{\"{o}}hler, J., and Welling, M.
\newblock Spherical cnns.
\newblock In \emph{6th International Conference on Learning Representations,
  {ICLR}}, 2018.

\bibitem[Dupont et~al.(2022)Dupont, Kim, Eslami, Rezende, and
  Rosenbaum]{dupont2022data}
Dupont, E., Kim, H., Eslami, S.~A., Rezende, D.~J., and Rosenbaum, D.
\newblock From data to functa: Your data point is a function and you can treat
  it like one.
\newblock In \emph{International Conference on Machine Learning}, pp.\
  5694--5725. PMLR, 2022.

\bibitem[Eilertsen et~al.(2020)Eilertsen, J{\"o}nsson, Ropinski, Unger, and
  Ynnerman]{eilertsen2020classifying}
Eilertsen, G., J{\"o}nsson, D., Ropinski, T., Unger, J., and Ynnerman, A.
\newblock Classifying the classifier: dissecting the weight space of neural
  networks.
\newblock In \emph{European Conference on Artificial Intelligence (ECAI 2020)},
  volume 325, pp.\  1119--1926, 2020.

\bibitem[Elesedy \& Zaidi(2021)Elesedy and Zaidi]{elesedy2021provably}
Elesedy, B. and Zaidi, S.
\newblock Provably strict generalisation benefit for equivariant models.
\newblock In \emph{International Conference on Machine Learning}, pp.\
  2959--2969. PMLR, 2021.

\bibitem[Entezari et~al.(2021)Entezari, Sedghi, Saukh, and
  Neyshabur]{entezari2021role}
Entezari, R., Sedghi, H., Saukh, O., and Neyshabur, B.
\newblock The role of permutation invariance in linear mode connectivity of
  neural networks.
\newblock In \emph{International Conference on Learning Representations}, 2021.

\bibitem[Esteves et~al.(2018)Esteves, Allen-Blanchette, Makadia, and
  Daniilidis]{esteves2018learning}
Esteves, C., Allen-Blanchette, C., Makadia, A., and Daniilidis, K.
\newblock Learning so (3) equivariant representations with spherical cnns.
\newblock In \emph{Proceedings of the European Conference on Computer Vision
  (ECCV)}, pp.\  52--68, 2018.

\bibitem[Finzi et~al.(2021)Finzi, Welling, and Wilson]{finzi2021practical}
Finzi, M., Welling, M., and Wilson, A.~G.
\newblock A practical method for constructing equivariant multilayer
  perceptrons for arbitrary matrix groups.
\newblock In \emph{International Conference on Machine Learning}, pp.\
  3318--3328. PMLR, 2021.

\bibitem[Fulton \& Harris(2013)Fulton and Harris]{fulton2013representation}
Fulton, W. and Harris, J.
\newblock \emph{Representation theory: a first course}, volume 129.
\newblock Springer Science \& Business Media, 2013.

\bibitem[Godfrey et~al.(2022)Godfrey, Brown, Emerson, and
  Kvinge]{godfrey2022symmetries}
Godfrey, C., Brown, D., Emerson, T., and Kvinge, H.
\newblock On the symmetries of deep learning models and their internal
  representations.
\newblock \emph{arXiv preprint arXiv:2205.14258}, 2022.

\bibitem[Hartford et~al.(2018)Hartford, Graham, Leyton-Brown, and
  Ravanbakhsh]{hartford2018deep}
Hartford, J., Graham, D., Leyton-Brown, K., and Ravanbakhsh, S.
\newblock Deep models of interactions across sets.
\newblock In \emph{International Conference on Machine Learning}, pp.\
  1909--1918. PMLR, 2018.

\bibitem[Hecht-Nielsen(1990)]{hecht1990algebraic}
Hecht-Nielsen, R.
\newblock On the algebraic structure of feedforward network weight spaces.
\newblock In \emph{Advanced Neural Computers}, pp.\  129--135. Elsevier, 1990.

\bibitem[Hornik(1991)]{hornik1991approximation}
Hornik, K.
\newblock Approximation capabilities of multilayer feedforward networks.
\newblock \emph{Neural networks}, 4\penalty0 (2):\penalty0 251--257, 1991.

\bibitem[Hubara et~al.(2016)Hubara, Courbariaux, Soudry, El-Yaniv, and
  Bengio]{hubara2016binarized}
Hubara, I., Courbariaux, M., Soudry, D., El-Yaniv, R., and Bengio, Y.
\newblock Binarized neural networks.
\newblock \emph{Advances in neural information processing systems}, 29, 2016.

\bibitem[Jaeckle \& Kumar(2021)Jaeckle and Kumar]{jaeckle2021generating}
Jaeckle, F. and Kumar, M.~P.
\newblock Generating adversarial examples with graph neural networks.
\newblock In \emph{Uncertainty in Artificial Intelligence}, pp.\  1556--1564.
  PMLR, 2021.

\bibitem[Keriven \& Peyr{\'e}(2019)Keriven and Peyr{\'e}]{keriven2019universal}
Keriven, N. and Peyr{\'e}, G.
\newblock Universal invariant and equivariant graph neural networks.
\newblock \emph{Advances in Neural Information Processing Systems}, 32, 2019.

\bibitem[Knyazev et~al.(2021)Knyazev, Drozdzal, Taylor, and
  Romero~Soriano]{knyazev2021parameter}
Knyazev, B., Drozdzal, M., Taylor, G.~W., and Romero~Soriano, A.
\newblock Parameter prediction for unseen deep architectures.
\newblock \emph{Advances in Neural Information Processing Systems},
  34:\penalty0 29433--29448, 2021.

\bibitem[Kondor \& Trivedi(2018)Kondor and Trivedi]{kondor2018generalization}
Kondor, R. and Trivedi, S.
\newblock On the generalization of equivariance and convolution in neural
  networks to the action of compact groups.
\newblock In \emph{International Conference on Machine Learning}, pp.\
  2747--2755. PMLR, 2018.

\bibitem[Krizhevsky et~al.(2009)Krizhevsky, Hinton,
  et~al.]{krizhevsky2009learning}
Krizhevsky, A., Hinton, G., et~al.
\newblock Learning multiple layers of features from tiny images.
\newblock Technical report, University of Toronto, 2009.

\bibitem[LeCun et~al.(1998)LeCun, Bottou, Bengio, and
  Haffner]{lecun1998gradient}
LeCun, Y., Bottou, L., Bengio, Y., and Haffner, P.
\newblock Gradient-based learning applied to document recognition.
\newblock \emph{Proceedings of the IEEE}, 86\penalty0 (11):\penalty0
  2278--2324, 1998.

\bibitem[Lee et~al.(2019)Lee, Lee, Kim, Kosiorek, Choi, and Teh]{lee2019set}
Lee, J., Lee, Y., Kim, J., Kosiorek, A., Choi, S., and Teh, Y.~W.
\newblock Set transformer: A framework for attention-based
  permutation-invariant neural networks.
\newblock In \emph{International conference on machine learning}, pp.\
  3744--3753. PMLR, 2019.

\bibitem[Lim et~al.(2022)Lim, Robinson, Zhao, Smidt, Sra, Maron, and
  Jegelka]{lim2022sign}
Lim, D., Robinson, J., Zhao, L., Smidt, T., Sra, S., Maron, H., and Jegelka, S.
\newblock Sign and basis invariant networks for spectral graph representation
  learning.
\newblock \emph{arXiv preprint arXiv:2202.13013}, 2022.

\bibitem[Litany et~al.(2022)Litany, Maron, Acuna, Kautz, Chechik, and
  Fidler]{litany2022federated}
Litany, O., Maron, H., Acuna, D., Kautz, J., Chechik, G., and Fidler, S.
\newblock Federated learning with heterogeneous architectures using graph
  hypernetworks.
\newblock \emph{arXiv preprint arXiv:2201.08459}, 2022.

\bibitem[Loshchilov \& Hutter(2019)Loshchilov and
  Hutter]{loshchilov2017decoupled}
Loshchilov, I. and Hutter, F.
\newblock Decoupled weight decay regularization.
\newblock In \emph{7th International Conference on Learning Representations,
  {ICLR}}, 2019.

\bibitem[Lu \& Kumar(2019)Lu and Kumar]{lu2019neural}
Lu, J. and Kumar, M.~P.
\newblock Neural network branching for neural network verification.
\newblock In \emph{International Conference on Learning Representations}, 2019.

\bibitem[Luigi et~al.(2023)Luigi, Cardace, Spezialetti, Ramirez, Salti, and
  di~Stefano]{luigi2023deep}
Luigi, L.~D., Cardace, A., Spezialetti, R., Ramirez, P.~Z., Salti, S., and
  di~Stefano, L.
\newblock Deep learning on implicit neural representations of shapes.
\newblock In \emph{The Eleventh International Conference on Learning
  Representations}, 2023.
\newblock URL \url{https://openreview.net/forum?id=OoOIW-3uadi}.

\bibitem[Maron et~al.(2019{\natexlab{a}})Maron, Ben-Hamu, Serviansky, and
  Lipman]{maron2019provably}
Maron, H., Ben-Hamu, H., Serviansky, H., and Lipman, Y.
\newblock Provably powerful graph networks.
\newblock \emph{Advances in neural information processing systems}, 32,
  2019{\natexlab{a}}.

\bibitem[Maron et~al.(2019{\natexlab{b}})Maron, Ben{-}Hamu, Shamir, and
  Lipman]{maron2018invariant}
Maron, H., Ben{-}Hamu, H., Shamir, N., and Lipman, Y.
\newblock Invariant and equivariant graph networks.
\newblock In \emph{7th International Conference on Learning Representations,
  {ICLR}}, 2019{\natexlab{b}}.

\bibitem[Maron et~al.(2019{\natexlab{c}})Maron, Fetaya, Segol, and
  Lipman]{maron2019universality}
Maron, H., Fetaya, E., Segol, N., and Lipman, Y.
\newblock On the universality of invariant networks.
\newblock In \emph{International conference on machine learning}, pp.\
  4363--4371. PMLR, 2019{\natexlab{c}}.

\bibitem[Maron et~al.(2020)Maron, Litany, Chechik, and
  Fetaya]{maron2020learning}
Maron, H., Litany, O., Chechik, G., and Fetaya, E.
\newblock On learning sets of symmetric elements.
\newblock In \emph{International Conference on Machine Learning}, pp.\
  6734--6744. PMLR, 2020.

\bibitem[Mildenhall et~al.(2021)Mildenhall, Srinivasan, Tancik, Barron,
  Ramamoorthi, and Ng]{mildenhall2021nerf}
Mildenhall, B., Srinivasan, P.~P., Tancik, M., Barron, J.~T., Ramamoorthi, R.,
  and Ng, R.
\newblock Nerf: Representing scenes as neural radiance fields for view
  synthesis.
\newblock \emph{Communications of the ACM}, 65\penalty0 (1):\penalty0 99--106,
  2021.

\bibitem[Morris et~al.(2019)Morris, Ritzert, Fey, Hamilton, Lenssen, Rattan,
  and Grohe]{morris2019weisfeiler}
Morris, C., Ritzert, M., Fey, M., Hamilton, W.~L., Lenssen, J.~E., Rattan, G.,
  and Grohe, M.
\newblock Weisfeiler and leman go neural: Higher-order graph neural networks.
\newblock In \emph{Proceedings of the AAAI conference on artificial
  intelligence}, volume~33, pp.\  4602--4609, 2019.

\bibitem[Morris et~al.(2021)Morris, Lipman, Maron, Rieck, Kriege, Grohe, Fey,
  and Borgwardt]{morris2021weisfeiler}
Morris, C., Lipman, Y., Maron, H., Rieck, B., Kriege, N.~M., Grohe, M., Fey,
  M., and Borgwardt, K.
\newblock Weisfeiler and leman go machine learning: The story so far.
\newblock \emph{arXiv preprint arXiv:2112.09992}, 2021.

\bibitem[Neyshabur et~al.(2015)Neyshabur, Salakhutdinov, and
  Srebro]{neyshabur2015path}
Neyshabur, B., Salakhutdinov, R.~R., and Srebro, N.
\newblock Path-sgd: Path-normalized optimization in deep neural networks.
\newblock \emph{Advances in neural information processing systems}, 28, 2015.

\bibitem[Park et~al.(2019)Park, Florence, Straub, Newcombe, and
  Lovegrove]{park2019deepsdf}
Park, J.~J., Florence, P., Straub, J., Newcombe, R., and Lovegrove, S.
\newblock Deepsdf: Learning continuous signed distance functions for shape
  representation.
\newblock In \emph{Proceedings of the IEEE/CVF conference on computer vision
  and pattern recognition}, pp.\  165--174, 2019.

\bibitem[Peebles et~al.(2022)Peebles, Radosavovic, Brooks, Efros, and
  Malik]{peebles2022learning}
Peebles, W., Radosavovic, I., Brooks, T., Efros, A.~A., and Malik, J.
\newblock Learning to learn with generative models of neural network
  checkpoints.
\newblock \emph{arXiv preprint arXiv:2209.12892}, 2022.

\bibitem[Pe{\~n}a et~al.(2022)Pe{\~n}a, Medeiros, Dubail, Aminbeidokhti,
  Granger, and Pedersoli]{Pea2022RebasinVI}
Pe{\~n}a, F. A.~G., Medeiros, H.~R., Dubail, T., Aminbeidokhti, M., Granger,
  E., and Pedersoli, M.
\newblock Re-basin via implicit sinkhorn differentiation.
\newblock \emph{arXiv preprint arXiv:2212.12042}, 2022.

\bibitem[Qi et~al.(2017)Qi, Su, Mo, and Guibas]{qi2017pointnet}
Qi, C.~R., Su, H., Mo, K., and Guibas, L.~J.
\newblock Pointnet: Deep learning on point sets for 3d classification and
  segmentation.
\newblock In \emph{Proceedings of the IEEE conference on computer vision and
  pattern recognition}, pp.\  652--660, 2017.

\bibitem[Ravanbakhsh et~al.(2017)Ravanbakhsh, Schneider, and
  Poczos]{ravanbakhsh2017equivariance}
Ravanbakhsh, S., Schneider, J., and Poczos, B.
\newblock Equivariance through parameter-sharing.
\newblock In \emph{International conference on machine learning}, pp.\
  2892--2901. PMLR, 2017.

\bibitem[Sch{\"u}rholt et~al.(2021)Sch{\"u}rholt, Kostadinov, and
  Borth]{schurholt2021self}
Sch{\"u}rholt, K., Kostadinov, D., and Borth, D.
\newblock Self-supervised representation learning on neural network weights for
  model characteristic prediction.
\newblock \emph{Advances in Neural Information Processing Systems},
  34:\penalty0 16481--16493, 2021.

\bibitem[Sch{\"u}rholt et~al.(2022{\natexlab{a}})Sch{\"u}rholt, Knyazev,
  Gir{\'o}-i Nieto, and Borth]{schurholthyper}
Sch{\"u}rholt, K., Knyazev, B., Gir{\'o}-i Nieto, X., and Borth, D.
\newblock Hyper-representations as generative models: Sampling unseen neural
  network weights.
\newblock \emph{arXiv preprint arXiv:2209.14733}, 2022{\natexlab{a}}.

\bibitem[Sch{\"u}rholt et~al.(2022{\natexlab{b}})Sch{\"u}rholt, Taskiran,
  Knyazev, Gir{\'o}-i Nieto, and Borth]{schurholtmodel}
Sch{\"u}rholt, K., Taskiran, D., Knyazev, B., Gir{\'o}-i Nieto, X., and Borth,
  D.
\newblock Model zoos: A dataset of diverse populations of neural network
  models.
\newblock \emph{arXiv preprint arXiv:2209.14764}, 2022{\natexlab{b}}.

\bibitem[Simsek et~al.(2021)Simsek, Ged, Jacot, Spadaro, Hongler, Gerstner, and
  Brea]{simsek2021geometry}
Simsek, B., Ged, F., Jacot, A., Spadaro, F., Hongler, C., Gerstner, W., and
  Brea, J.
\newblock Geometry of the loss landscape in overparameterized neural networks:
  Symmetries and invariances.
\newblock In \emph{International Conference on Machine Learning}, pp.\
  9722--9732. PMLR, 2021.

\bibitem[Singh \& Jaggi(2020)Singh and Jaggi]{singh2020model}
Singh, S.~P. and Jaggi, M.
\newblock Model fusion via optimal transport.
\newblock \emph{Advances in Neural Information Processing Systems},
  33:\penalty0 22045--22055, 2020.

\bibitem[Sitzmann et~al.(2020)Sitzmann, Martel, Bergman, Lindell, and
  Wetzstein]{sitzmann2020implicit}
Sitzmann, V., Martel, J., Bergman, A., Lindell, D., and Wetzstein, G.
\newblock Implicit neural representations with periodic activation functions.
\newblock \emph{Advances in Neural Information Processing Systems},
  33:\penalty0 7462--7473, 2020.

\bibitem[Tancik et~al.(2020)Tancik, Srinivasan, Mildenhall, Fridovich-Keil,
  Raghavan, Singhal, Ramamoorthi, Barron, and Ng]{tancik2020fourier}
Tancik, M., Srinivasan, P., Mildenhall, B., Fridovich-Keil, S., Raghavan, N.,
  Singhal, U., Ramamoorthi, R., Barron, J., and Ng, R.
\newblock Fourier features let networks learn high frequency functions in low
  dimensional domains.
\newblock \emph{Advances in Neural Information Processing Systems},
  33:\penalty0 7537--7547, 2020.

\bibitem[Tatro et~al.(2020)Tatro, Chen, Das, Melnyk, Sattigeri, and
  Lai]{tatro2020optimizing}
Tatro, N., Chen, P.-Y., Das, P., Melnyk, I., Sattigeri, P., and Lai, R.
\newblock Optimizing mode connectivity via neuron alignment.
\newblock \emph{Advances in Neural Information Processing Systems},
  33:\penalty0 15300--15311, 2020.

\bibitem[Thomas et~al.(2018)Thomas, Smidt, Kearnes, Yang, Li, Kohlhoff, and
  Riley]{thomas2018tensor}
Thomas, N., Smidt, T., Kearnes, S., Yang, L., Li, L., Kohlhoff, K., and Riley,
  P.
\newblock Tensor field networks: Rotation-and translation-equivariant neural
  networks for 3d point clouds.
\newblock \emph{arXiv preprint arXiv:1802.08219}, 2018.

\bibitem[Unterthiner et~al.(2020)Unterthiner, Keysers, Gelly, Bousquet, and
  Tolstikhin]{unterthiner2020predicting}
Unterthiner, T., Keysers, D., Gelly, S., Bousquet, O., and Tolstikhin, I.
\newblock Predicting neural network accuracy from weights.
\newblock \emph{arXiv preprint arXiv:2002.11448}, 2020.

\bibitem[Vaswani et~al.(2017)Vaswani, Shazeer, Parmar, Uszkoreit, Jones, Gomez,
  Kaiser, and Polosukhin]{vaswani2017attention}
Vaswani, A., Shazeer, N., Parmar, N., Uszkoreit, J., Jones, L., Gomez, A.~N.,
  Kaiser, {\L}., and Polosukhin, I.
\newblock Attention is all you need.
\newblock \emph{Advances in neural information processing systems}, 30, 2017.

\bibitem[Velickovic et~al.(2018)Velickovic, Cucurull, Casanova, Romero,
  Li{\`{o}}, and Bengio]{velivckovic2017graph}
Velickovic, P., Cucurull, G., Casanova, A., Romero, A., Li{\`{o}}, P., and
  Bengio, Y.
\newblock Graph attention networks.
\newblock In \emph{6th International Conference on Learning Representations,
  {ICLR}}, 2018.

\bibitem[Wang et~al.(2022)Wang, Wang, Liang, and Lai]{wang2022understanding}
Wang, G., Wang, G., Liang, W., and Lai, J.
\newblock Understanding weight similarity of neural networks via chain
  normalization rule and hypothesis-training-testing.
\newblock \emph{arXiv preprint arXiv:2208.04369}, 2022.

\bibitem[Wang et~al.(2019)Wang, Yurochkin, Sun, Papailiopoulos, and
  Khazaeni]{wang2019federated}
Wang, H., Yurochkin, M., Sun, Y., Papailiopoulos, D., and Khazaeni, Y.
\newblock Federated learning with matched averaging.
\newblock In \emph{International Conference on Learning Representations}, 2019.

\bibitem[Wang et~al.(2020)Wang, Albooyeh, and Ravanbakhsh]{wang2020equivariant}
Wang, R., Albooyeh, M., and Ravanbakhsh, S.
\newblock Equivariant networks for hierarchical structures.
\newblock \emph{Advances in Neural Information Processing Systems},
  33:\penalty0 13806--13817, 2020.

\bibitem[Wood \& Shawe-Taylor(1996)Wood and
  Shawe-Taylor]{wood1996representation}
Wood, J. and Shawe-Taylor, J.
\newblock Representation theory and invariant neural networks.
\newblock \emph{Discrete applied mathematics}, 69\penalty0 (1-2):\penalty0
  33--60, 1996.

\bibitem[Xiao et~al.(2017)Xiao, Rasul, and Vollgraf]{xiao2017fashion}
Xiao, H., Rasul, K., and Vollgraf, R.
\newblock Fashion-mnist: a novel image dataset for benchmarking machine
  learning algorithms.
\newblock \emph{arXiv preprint arXiv:1708.07747}, 2017.

\bibitem[Xu et~al.(2022)Xu, Wang, Jiang, Fan, and Wang]{xu2022signal}
Xu, D., Wang, P., Jiang, Y., Fan, Z., and Wang, Z.
\newblock Signal processing for implicit neural representations.
\newblock In \emph{Advances in Neural Information Processing Systems}, 2022.

\bibitem[Xu et~al.(2019)Xu, Hu, Leskovec, and Jegelka]{xu2018powerful}
Xu, K., Hu, W., Leskovec, J., and Jegelka, S.
\newblock How powerful are graph neural networks?
\newblock In \emph{7th International Conference on Learning Representations,
  {ICLR}}, 2019.

\bibitem[Yurochkin et~al.(2019)Yurochkin, Agarwal, Ghosh, Greenewald, Hoang,
  and Khazaeni]{yurochkin2019bayesian}
Yurochkin, M., Agarwal, M., Ghosh, S., Greenewald, K., Hoang, N., and Khazaeni,
  Y.
\newblock Bayesian nonparametric federated learning of neural networks.
\newblock In \emph{International Conference on Machine Learning}, pp.\
  7252--7261. PMLR, 2019.

\bibitem[Zaheer et~al.(2017)Zaheer, Kottur, Ravanbakhsh, Poczos, Salakhutdinov,
  and Smola]{zaheer2017deep}
Zaheer, M., Kottur, S., Ravanbakhsh, S., Poczos, B., Salakhutdinov, R.~R., and
  Smola, A.~J.
\newblock Deep sets.
\newblock \emph{Advances in neural information processing systems}, 30, 2017.

\end{thebibliography}
\bibliographystyle{icml2023}

\newpage
\appendix
\onecolumn

\section{Related Work}\label{sec:more_prev}

\textbf{Processing neural networks.} 
In recent years several studies suggested using the parameters of NNs for learning tasks. \citet{BakerGRN18} tries to infer the final performance of a model based on plain statistics such as the network architecture, validation accuracy at different checkpoints, and hyper-parameters. In a similar vein, both \cite{eilertsen2020classifying, unterthiner2020predicting} attempt to predict properties of trained NNs based on their weights. \cite{eilertsen2020classifying} tries to predict the hyper-parameters used to train the network, and \cite{unterthiner2020predicting} tries to predict the network generalization capabilities. Both of these studies use standard NNs on the flattened weights or on some statistics of them. Our approach, on the other hand, introduces useful inductive biases for these learning tasks and is not limited to the scope of these studies.
In \cite{xu2022signal}, it was proposed that neural networks can be processed by applying a neural network to a concatenation of their high-order spatial derivatives. The method focuses on INRs, for which derivative information is relevant, and depends on the ability to sample the input space efficiently. The ability of these networks to handle more general tasks is still not well understood. Furthermore, these architectures may require high-order derivatives, which result in a substantial computational burden. 
\citet{dupont2022data} suggested applying deep learning tasks, such as generative modeling, to a dataset of INRs fitted from the original data. To obtain useful representations of the data, the authors suggest to meta-learn low dimensional vectors, termed modulations, which are embedded in a NN with shared parameters across all training examples. Unlike this approach, our method can work on any network and is agnostic to the way that it was trained. Several studies \cite{lu2019neural, jaeckle2021generating, knyazev2021parameter,litany2022federated} treated the NNs as graphs for formal verification, generating adversarial examples, and parameter prediction respectively.  \citet{peebles2022learning} proposed a generative approach to output a target network based on an initial network and a target metric such as the loss value or return. \citet{schurholtmodel} published a dataset of vectorized trained neural networks, referred to as model-zoo, to encourage research on NN models. Since these models have a CNN architecture, they are not suitable for us. 
In \cite{schurholt2021self} the authors suggest methods to learn representations of trained NNs using self-supervised methods, and in \cite{schurholthyper} this approach is leveraged for NN model generation. The empirical evaluation in the paper shows that our method compares favorably to this baselines. Similar modeling was utilized in a recent submission by \citet{luigi2023deep}. In this study, the authors propose a methodology for processing of INRs that combines two components: (1) a neural architecture that operates on stacks of weights and bias vectors assuming a set structure, and (2) a pre-training procedure based on task ensuring that the output of this network is capable of reconstructing the INR. 
It should be noted that this work (1) relies on the ability to evaluate the INR as a function, which is feasible only in low dimensional spaces; and (2) assumes all data was generated using a meta-learning algorithm so that their representations would be aligned.
Moreover, from a symmetry and equivariance perspective, 
their formulation assumes that the rows of all weight matrices and all biases have a \emph{global} set structure, which implies that their networks are invariant to permutations of rows and biases across weight matrices.  Unfortunately, in general, such permutations could result in a change in the underlying function. Therefore, from a symmetry and equivariance perspective, their work improperly models the symmetry group. 
Finally, in recent years several studies inspected the problem of aligning the weights of NNs \cite{ashmore2015method, yurochkin2019bayesian, wang2019federated, singh2020model, tatro2020optimizing, entezari2021role, ainsworth2022git, wang2022understanding}. As stated in the main text, solving the alignment tasks is hard and these strategies suffer from scaling issues to large datasets.




\textbf{Equivariant architectures.} 
Complex data types, such as graphs and images, are often associated with groups of transformations that change data representation without changing the underlying data. These groups are known as symmetry groups, and they are commonly formulated through group representations.
Functions defined on these objects are often invariant or equivariant to these symmetry transformations. 
A good example of this would be a graph classification function that is node-permutation invariant, or an image segmentation function that is translation equivariant. 
When trying to learn such functions, a wide range of studies have demonstrated that constraining learning models to be equivariant or invariant to these transformations has many advantages, including smaller parameter space, efficient implementation, and better generalization abilities \cite{cohen2018spherical,kondor2018generalization,esteves2018learning,zaheer2017deep,hartford2018deep,maron2018invariant,elesedy2021provably}. 
The majority of equivariant and invariant models are constructed in the same manner: first, a simple equivariant function is identified. In many cases, these are linear \cite{zaheer2017deep,hartford2018deep,maron2018invariant}, although they may also be non-linear \cite{maron2019provably,thomas2018tensor, azizian2020expressive}. The network is then constructed by composing these simple functions interleaved with pointwise nonlinear functions. This paradigm was successfully applied to a multitude of data types, from graphs and sets \cite{zaheer2017deep,maron2018invariant}, through 3D data \cite{esteves2018learning} and spherical functions \cite{cohen2018spherical} to images \cite{cohen2016group}.

\textbf{Spaces of linear equivariant layers.} For a group $G$ and representation $(\mathcal{V},\rho),(\mathcal{W},\rho')$, solving for the space of linear equivariant layers $L:\mathcal{V}\rightarrow \mathcal{W}$ amounts to solving a system of linear equations of the form $L\rho(g)=\rho'(g)L$ for all $g\in G$, where $L$ is our unknown equivariant layer. \citet{wood1996representation,ravanbakhsh2017equivariance,maron2018invariant} showed that if $G$ is a finite permutation group, and $\rho,\rho'$ are permutation representations, then a basis for the space of equivariant maps is spanned by indicator tensors for certain orbits of the group action. Alternatively, \cite{finzi2021practical} derived numerical algorithms for solving these systems of equations.

\textbf{Learning on set-structured data.} Among the most prominent examples of equivariant architectures are those designed to process set-structured data, where the input represents a set of elements and the learning tasks are invariant or equivariant to their order. The pioneering works in this area were DeepSets \cite{zaheer2017deep} and PointNet \cite{qi2017pointnet}. 
In subsequent work, the linear sum aggregation has been replaced with attention mechanisms \cite{lee2019set} and the layer characterization has been extended to multiple sets \cite{hartford2018deep} and sets with structured elements \cite{maron2020learning,wang2020equivariant}. 
As shown in Section \ref{sec:sym}, our weight-space symmetry group is a product of symmetric groups acting by permuting the weight spaces. A key observation we make in Section \ref{sec:charcterize} is that our basic linear layer can be broken up into multiple linear blocks that implement previously characterized equivariant layers for sets.  

Here we define the layers from \cite{zaheer2017deep,hartford2018deep} as they play a significant role in our DWS-layers.
\textit{DeepSets \cite{zaheer2017deep}}: For an input $X\in \mathbb{R}^{n\times d}$, that represents a set of $n$ elements, the DeepSets layer is the most general $S_n$-equivariant linear layer and is defined as $L_{\text{DS}}(X)_i=W_1 X_i + W_2\sum_j X_j$, where $W_1,W_2\in \mathbb{R}^{d'\times d}$ are learnable linear transformations. (ii) \textit{Equivariant layers for multiple sets}:  these are layers for  cases where the input involves two or more set dimensions. Formally, let $X\in \mathbb{R}^{n\times m \times d}$ where $m,n$ represent set dimensions, meaning we don't care about the order of the elements in these dimensions, and $d$ is the number of feature channels. \citet{hartford2018deep} showed that the most general $S_n\times S_m$-equivariant  linear layer is of the form  $L_{\text{Har}}(X)_{ij}=W_1 X_{ij} + W_2 \sum_i X_{ij} + W_3 \sum_j X_{ij} + W_4 \sum_{ij} X_{ij}$, where, again $W_1,W_2,W_3,W_4\in \mathbb{R}^{d'\times d}$.

\section{Multiple Channels, Invariant Layers and Biases for Equivariant Maps} \label{appx:features_and_bias}
Here we discuss equivariant maps between weight spaces with multiple features and bias terms.

\textbf{Layers with multiple feature channels.} 
It is common for deep networks to represent their input objects using multiple feature channels. Equivariant layers for multiple input and output channels can be obtained by using Proposition \ref{proposition:equi_direct_sum}. Formally, let $\mathcal{L}$ be a space of linear $G$-equivariant maps from $\mathcal{U}$ to $\mathcal{U}'$. A higher dimensional feature space for $\mathcal{U},\mathcal{U}'$ can be formulated as a direct sum of multiple copies of these spaces. A general linear equivariant map $L:\mathcal{U}^f\rightarrow \mathcal{U}'^{f'}$, where $f,f'$ are the feature dimensions, can be written as $L(X)_j = \sum_{i=1}^{f} L_{ij}(x_i)$, where $x_i$ refers to the $i$-th representation in the direct sum and $L_{ij}\in \mathcal{L}$ \footnote{See \cite{maron2018invariant} for a different way of deriving that.}.  

\textbf{Biases.} One typically adds a constant bias term to each output channel of the linear equivariant maps derived in Theorem \ref{thm:char_equi} to create affine transformations. As mentioned in \cite{maron2018invariant}, these bias terms have to obey a set of equations to make sure they are equivariant: if $L(X)=b\in \mathcal{V}$ is a constant map then we have $b=L(\rho(g)x)=\rho(g)L(x)=\rho(g)b$. When $\rho$ is a permutation representation, this means that the bias vector is constant on the orbits of the permutation group acting on the indices of the vector, leading to the following characterization:  

\begin{proposition}\label{lemma:bias}
    Let $G\leq S_n$ be a permutation group and $P$ its permutation representation on $\mathbb{R}^n$. Any vector $b\in\mathbb{R}^n$ with the property $b=P(g)b$ for all $g\in G$ is of the form $b = \sum_{i=1}^O w_ia_i$ where $w_i$ are scalars, $a_i\in \mathbb{R}^n,$ are indicators of the orbits of the action of $G$ on $[n]$  and $O$ is the number of such orbits.
\end{proposition}

In our case, we can think of $G$ as a subgroup of the permutation group on the indices of $\mathcal{V}$, i.e.,  all the entries of the weights and biases of an input network. The orbits of $G$, in that case, are subsets of the indices associated with specific weight and bias spaces, $\{ \mathcal{W}_m,\,\mathcal{B}_\ell\}$, and we can list them separately for each bias of weight space. 
Table \ref{tab:orbits} lists these orbits. As an example, the bias term corresponds to $W_i$ for $i\in[2, M-1]$ is constant matrix $w\cdot11^T$ for a learnable scalar $w$; The bias term that corresponds to $W_1$ is constant along the columns, and the bias term that corresponds to $W_M$ is constant along the rows. Effectively, the complete bias term for $\mathcal{V}$ is a concatenation of the bias terms for all weights and biases spaces. 


\paragraph{Linear Invariant Maps for Weight-Spaces.}
Here, we provide a characterization of linear $G$-invariant maps $L:\mathcal{V}\rightarrow \mathbb{R}$. Invariant layers (which are often followed by fully connected networks) are typically placed after a composition of several equivariant layers when the task at hand requires a single output, e.g., when the input network represents an INR of a 3D shape and the task is to classify the shapes. 
We use the following characterization of linear invariant maps from \citet{maron2018invariant}:

\begin{proposition}\label{lemma:inv}
    Let $G\leq S_n$ be a permutation group and $P$ its permutation representation on $\mathbb{R}^n$. Every linear $G$-invariant map $L:\mathbb{R}^n\rightarrow \mathbb{R}$ is of the form $L(x) = \sum_{i=1}^O w_ia_i^Tx$ where $w_i$ are learnable scalars, $a_i\in \mathbb{R}^n$ are indicator vectors for the orbits of the action of $G$ on $[n]$ and $O$ is the number of such orbits.
\end{proposition}
This proposition follows directly from the fact that a weight vector $w$ has to obey the equation  $w=\rho(g)w$ for all group elements $g\in G$.
In our case, $G$ is a permutation group acting on the index space of $\mathcal{V}$, i.e., the indices of all the weights and biases of an input network. In order to apply \cref{lemma:inv}, we  need to find the orbits of this action on the indices of $\mathcal{V}$. Importantly, each such orbit is a subset of the indices that correspond to a specific weight or bias vector. These orbits are summarized in Table \ref{tab:orbits}. 
It follows that every linear invariant map defined on $\mathcal{V}$ can be written as a summation of the maps listed below: (1) a distinct learnable scalar times the sum of $W_m$ for $m\in[2, M-1]$ and the sum of $b_m$ for  $m\in[1, M-1]$; (2) a sum of columns of $W_1$, and the sum of rows of $W_M$ weighted by distinct learnable scalars for each such column and row (3) an inner product of $b_M$ with a learnable vector of size $d_M$.


\section{Specification of All Affine Equivariant Layers Between Sub-Representations}\label{app:tables}
Tables \ref{tab:w2w}, \ref{tab:b2b}, \ref{tab:w2b}, \ref{tab:b2w} specify the implementation and dimensionality of all the types of equivariant maps between the sub-representations $\{\mathcal{W}_m,\mathcal{B}_\ell\}_{m,\ell\in[M]}$.The indices $i,j$ represent  the indices of the blocks. Dimensions in $LIN, L_{DS}, L_{Har}$ layers specify input and output dimensions. $L_{DS},L_{Har}$ are formally defined in Appendix \ref{sec:more_prev}.
Layers marked with an asterisk symbol (*)  have the same layer type at a different position in the block matrix.

\begin{figure*}[t!]
    \centering
    \includegraphics[width=\linewidth]{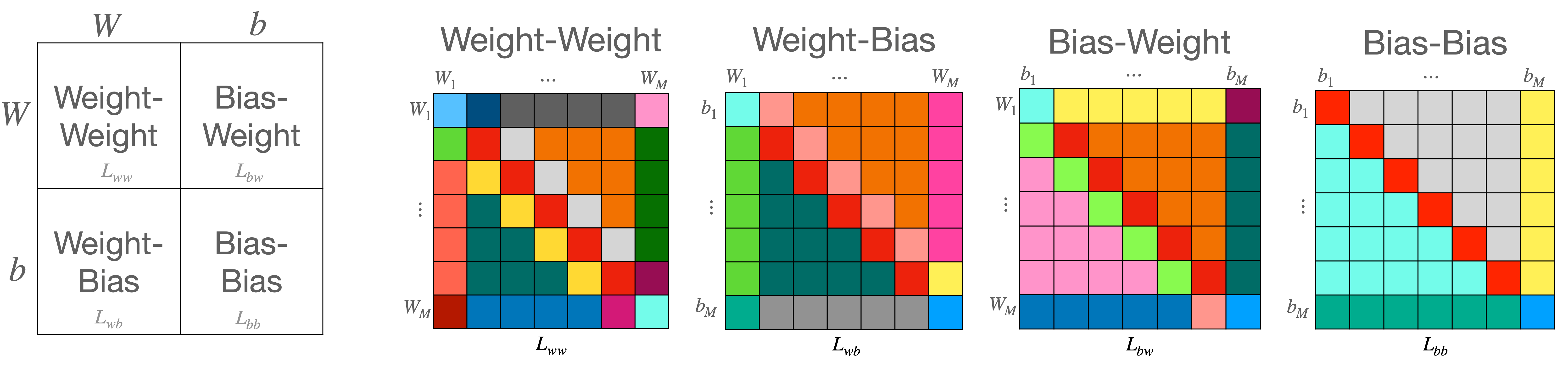}
    \caption{Block matrix structure for linear equivariant maps between weight spaces (same as in the main paper). Left: an equivariant layer for the weight space $\mathcal{V}$ to itself can be written as four blocks that map between the general weight space $\mathcal{W}$ and general bias space $\mathcal{B}$. Right: Each such block can be further written as a block matrix between specific weight and bias spaces $\mathcal{W}_m,\mathcal{B}_\ell$. Each color in each block matrix represents a different type of linear equivariant function between the sub-representations $\mathcal{W}_m,\mathcal{B}_\ell$. Repeating colors in different matrices are not related. See Tables \ref{tab:w2w}-\ref{tab:b2w} for a specification of the layers.
    }
    \label{fig:blocks_sub}
\end{figure*}

\vspace{-0pt}

\begin{table*} [h!]
  \centering
  \caption{Specification of layers in the weight-to-weight block.}
  \vskip 0.11in
  \tiny
  \resizebox{\textwidth}{!}{
    \begin{tabular}{llllllll}
    \hline
          & color &  & condition & sub condition & from space to space ($\mathcal{W}_j\rightarrow\mathcal{W}_i$) & implementation & \# params \\
          \hline
    \multirow{3}[0]{*}{Diagonal} & \cellcolor[rgb]{ .337,  .757,  1} & 1     & $j=i$ & $i=j=1$ & $d_1  \times  d_0 \rightarrow d_1  \times  d_0$ & $L_{\text{DS}}(d_0,d_0)$ & $2d_0^2$ \\
          & \cellcolor[rgb]{ .455,  .992,  .914} & 2     &       & $i=j=M$ & $d_{M}  \times  d_{M-1} \rightarrow d_{M}  \times  d_{M-1} $ & $L_{\text{DS}}(d_M,d_M)$ & $2d_M^2$ \\
          & \cellcolor[rgb]{ .933,  .137,  .047} & 3     &       & $1<i=j<M$ & $d_{i}  \times  d_{i-1} \rightarrow d_{i}  \times  d_{i-1} $ & $L_{\text{Har}}(d_i,d_{i-1})$ & $4$ \\
          \hline
    \multirow{3}[0]{*}{One above diagonal} & \cellcolor[rgb]{ .012,  .298,  .498} & 4     & $j=i+1$ & $i=1$ & $d_2  \times  d_1 \rightarrow d_1  \times  d_0$ & $POOL(d_2)  \rightarrow  L_{\text{DS}}(1,d0)$ & $2d_0$ \\
          & \cellcolor[rgb]{ .592,  .055,  .322} & 5     &       & $i=M-1$ & $d_{M}  \times  d_{M-1} \rightarrow d_{M-1}  \times  d_{M-2}$ & $L_{\text{DS}}(d_{M},1)  \rightarrow  BC(d_{M-2})$ & $2d_{M}$ \\
          & \cellcolor[rgb]{ .835,  .835,  .839} & 6    &       & $1<i<M-1$ & $d_{i+1}  \times  d_{i} \rightarrow d_{i}  \times  d_{i-1}$ & $POOL(d_{i+1})  \rightarrow  L_{\text{DS}}(1,1)  \rightarrow  BC(d_{i-1})$ & $2$ \\
          \hline
    \multirow{3}[0]{*}{One below diagonal} & \cellcolor[rgb]{ .384,  .851,  .212} & 7     & $j=i-1$ & $i=1$ & $d_1  \times  d_0 \rightarrow d_2  \times  d_1$ & $L_{\text{DS}}(d_0,1) \rightarrow BC(d_2)$ & $2d_0$ \\
          & \cellcolor[rgb]{ .831,  .09,  .467} & 8     &       & $i=M-1$ & $d_{M-1}  \times  d_{M-2} \rightarrow d_{M}  \times  d_{M-1}$ & $POOL(d_{m-2})  \rightarrow  L_{\text{DS}}(1,d_M)$ &  $2d_M$ \\
          & \cellcolor[rgb]{ 1,  .855,  .192} & 9    &       & $1<i<M-1$ & $d_{i-1}  \times  d_{i-2} \rightarrow d_{i}  \times  d_{i-1}$ & $POOL(d_{i-2})  \rightarrow  L_{\text{DS}}(1,1)  \rightarrow  BC(d_{i})$ & $2$ \\
          \hline
    \multirow{4}[0]{*}{Upper triangular} & \cellcolor[rgb]{ .369,  .369,  .369} & 10    & $j>i+1$ & $i=1  \text{ and }  j<M$ & $d_{j}  \times  d_{j-1} \rightarrow  d_1  \times  d_0$ & $POOL(d_j,d_{j-1}) \rightarrow  LIN(1,d_0) \rightarrow BC(d_1)$ & $d_0$ \\
          & \cellcolor[rgb]{ 1.,  .6,  .8} & 11    &       & $i=1  \text{ and }  j=M$ & $d_{M}  \times  d_{M-1} \rightarrow  d_1  \times  d_0$ & $POOL(d_m-1)  \rightarrow  LIN(d_M,d_0)  \rightarrow  BC(d_1)$ & $d_0d_M$ \\
          & \cellcolor[rgb]{ 0,  .443,  .004} & 12    &       & $i>1  \text{ and }  j=M$ & $d_{M}  \times  d_{M-1} \rightarrow  d_{i}  \times  d_{i-1}$ & $POOL(d_{M-1}) \rightarrow  LIN(d_M,1) \rightarrow BC(d_{i},d_{i-1})$ & $d_M$ \\
          & \cellcolor[rgb]{ .953,  .443,  .016}\textcolor[rgb]{ .953,  .443,  .016}{} & 13*  &       & $i>1  \text{ and }  j<M$ & $d_{j}  \times  d_{j-1} \rightarrow  d_{i}  \times  d_{i-1}$ & $POOL(d_j,d_{j-1}) \rightarrow  LIN(1,1) \rightarrow BC(d_i,d_{i-1})$ & $1$ \\
          \hline
    \multirow{4}[0]{*}{Lower triangular} & \cellcolor[rgb]{ 1,  .392,  .306} & 14    & $j<i-1$ & $j=1  \text{ and }  i<M$ & $d_1  \times  d_0 \rightarrow d_{i}  \times  d_{i-1}$ & $POOL(d1)\rightarrow LIN(d_0,1) \rightarrow BC(d_{i-1},d_i)$ & $d_0$ \\
          & \cellcolor[rgb]{ .714,  .082,  0} & 15    &       & $j=1  \text{ and }  i=M$ & $d_1  \times  d_0 \rightarrow d_{M}  \times  d_{M-1}$ & $POOL(d_1)  \rightarrow  LIN(d_0,d_M)  \rightarrow  BC(d_{M-1})$ & $d_0d_M$ \\
          & \cellcolor[rgb]{ 0,  .463,  .733} & 16    &       & $j>1  \text{ and }  i=M$ & $d_{j}  \times  d_{j-1} \rightarrow d_{M}  \times  d_{M-1}$ & $POOL(d_{j},d_{j-1})\rightarrow LIN(1,d_M)\rightarrow BC(d_{M-1} )$ & $d_M$ \\
          & \cellcolor[rgb]{ 0,  .424,  .396} & 17*  &       & $j>1  \text{ and }  i<M$ & $d_{j}  \times  d_{j-1} \rightarrow  d_{i}  \times  d_{i-1}$ & $POOL(d_j,d_{j-1}) \rightarrow  LIN(1,1) \rightarrow BC(d_{i},d_{i-1})$ & $1$ \\
          \hline
    \end{tabular}%
    }
\label{tab:w2w}
\end{table*}%

\vspace{-0pt}

\begin{table*}[h!]
  \centering
  \tiny
  \caption{Specification of layers in the bias-to-bias block.}
 \vskip 0.11in
    \begin{tabular}{llllllll}
              \hline
          & color &       & condition & sub condition & from space to space ($\mathcal{B}_j\rightarrow\mathcal{B}_i$) & implementation &\# params \\
          \hline

    \multirow{2}[0]{*}{Diagonal} & \cellcolor[rgb]{ 1,  .149,  0} & 1     & $i=j$ & $i=j<M$ & $d_i \rightarrow d_i$ & $L_{\text{DS}}(1,1)$ & $2$ \\
          & \cellcolor[rgb]{ .024,  .631,  1} & 2     &       & $i=j=M$ & $d_M\rightarrow d_M$ & $LIN(d_M,d_M)$ & $d_M^2$ \\
          \hline
    \multirow{2}[0]{*}{Upper triangular} & \cellcolor[rgb]{ 1,  .949,  .337} & 3     & $i<j$ & $j=M$ & $d_j\rightarrow d_i$ & $LIN(d_M,1)\rightarrow BC(d_i)$ & $d_M$ \\
          & \cellcolor[rgb]{ .835,  .835,  .835} & 4*    &       & $j<M$ & $d_j\rightarrow d_i$ & $POOL(d_j)\rightarrow LIN(1,1)\rightarrow BC(d_i)$ & $1$ \\
                    \hline
    \multirow{2}[0]{*}{Lower triangular} & \cellcolor[rgb]{ .024,  .671,  .561} & 5     & $i>j$ & $i=M$ & $d_j\rightarrow d_i$ & $POOL(d_j)\rightarrow LIN(1,d_M)$ & $d_M$ \\
          & \cellcolor[rgb]{ .447,  .996,  .918} & 6*     &       & $i<M$ & $d_j\rightarrow d_i$ & $POOL(d_j)\rightarrow LIN(1,1) \rightarrow BC(d_i)$ & $1$ \\
                    \hline

    \end{tabular}%
\label{tab:b2b}
\end{table*}
\vspace{-20pt}
\begin{table*}[h!]
  \centering
  \tiny
  \caption{Specification of layers in the weight-to-bias block.}
\vskip 0.11in
\resizebox{\textwidth}{!}{
    \begin{tabular}{llllllll}
              \hline

          & color &       & condition & sub condition & from space to space ($\mathcal{W}_j\rightarrow\mathcal{B}_i$) & implementation & \# params \\
              \hline
    \multirow{3}[0]{*}{Diagonal} & \cellcolor[rgb]{ .443,  .992,  .914} & 1     & $i=j$ & $i=j=1$ & $d_1  \times  d_0 \rightarrow  d_1$ & $L_{\text{DS}}(d_0,1) $ & $2d_0$ \\
          & \cellcolor[rgb]{ .937,  .129,  .047} & 2    &       & $1<i=j<M$ & $d_{i}  \times  d_{i-1} \rightarrow  d_i$ & $POOL(d_{i-1})\rightarrow  L_{\text{DS}}(1,1)$ & $2$ \\
          & \cellcolor[rgb]{ .024,  .631,  1} & 3     &       & $i=j=M$ & $d_{M}  \times  d_{M-1} \rightarrow  d_M$ & $POOL(d_{M-1})\rightarrow  LIN(d_M,d_M)$ & $d_M^2$ \\
                    \hline
    \multirow{2}[0]{*}{One above diagonal} & \cellcolor[rgb]{ 1,  .588,  .549} & 4    & $j=i+1$ & $j<M$ & $d_{j}  \times  d_{j-1} \rightarrow  d_{j-1}$ & $POOL(d_{j})\rightarrow  L_{\text{DS}}(1,1)$ & $2$ \\
          & \cellcolor[rgb]{ 1,  .945,  .337} & 5     &       & $j=M$ & $d_{M}  \times  d_{M-1} \rightarrow  d_{M-1}$ &  $L_{\text{DS}}(d_{M},1)$ & $2d_M$ \\
              \hline
    \multirow{2}[0]{*}{Upper triangular} & \cellcolor[rgb]{ .953,  .447,  0} & 6*   & $j>i+1$ & $j<M$ & $d_{j}  \times  d_{j-1} \rightarrow  d_i$ & $POOL(d_{j-1},d_j)\rightarrow LIN(1,1)\rightarrow BC(d_1)$ & $1$ \\
          & \cellcolor[rgb]{ 1,  .259,  .627} & 7     &       &  $j=M$ & $d_{M}  \times  d_{M-1} \rightarrow  d_i$ & $POOL(d_{M-1})\rightarrow  LIN(d_M,1)\rightarrow BC(d_i)$ & $d_M$ \\
              \hline

    \multirow{2}[0]{*}{Lower triangular} & \cellcolor[rgb]{ .376,  .851,  .212} & 8     & $j=i-1$ & $j=1 \text{ and }  i<M$ & $d_1  \times  d_0 \rightarrow  d_i$ & $POOL(d_1)+LIN(d_0,1)\rightarrow BC(d_i)$ & $d_0$ \\
          & \cellcolor[rgb]{ .024,  .671,  .561} & 9     & $j<i-1$ & $j=1 \text{ and } i=M$ & $d_1  \times  d_0 \rightarrow  d_M$ & $POOL(d_1)+LIN(d_0,d_M)$ & $d_0d_M$ \\
          & \cellcolor[rgb]{ .573,  .573,  .569} & 10    &       & $j>1 \text{ and }i=M$ & $d_{j}  \times  d_{j-1} \rightarrow  d_M$ & $POOL(d_{j-1},d_j)\rightarrow LIN(1,d_M)$ & $d_M$ \\
          & \cellcolor[rgb]{ 0,  .424,  .396} & 11*  &       & $j>1 \text{ and }i<M$ & $d_{j}  \times  d_{j-1} \rightarrow  d_i$ & $POOL(d_{j-1},d_j)\rightarrow LIN(1,1)\rightarrow BC(d_i)$ & $1$ \\
          \hline

    \end{tabular}%
    }
\label{tab:w2b}
\end{table*}%

\vspace{-10pt}

\begin{table*}[h!]
  \centering
  \tiny
  
  \caption{Specification of layers in the bias-to-weight block.}
  \vskip 0.11in
  \resizebox{\textwidth}{!}{
    \begin{tabular}{llllllll}
              \hline

          & color &       & condition & sub condition & from space to space $\mathcal{B}_j\rightarrow\mathcal{W}_i$ & implementation & \# params \\
          \hline

    \multirow{3}[0]{*}{Diagonal} & \cellcolor[rgb]{ .451,  .988,  .918} & 1     & $i=j$ & $i=j=1$ & $d_1 \rightarrow  d_1  \times  d_0$ & $L_{\text{DS}}(1,d_0)$ & $2d_0$ \\
          & \cellcolor[rgb]{ .933,  .133,  .051} & 2    &       & $1<i=j<M$ & $d_i \rightarrow  d_{i}  \times  d_{i-1}$ & $L_{\text{DS}}(1,1)\rightarrow BC(d_{i-1})$ & $2$ \\
          & \cellcolor[rgb]{ .22,  .631,  .996} & 3     &       & $i=j=M$ & $d_M \rightarrow  d_{M}  \times  d_{M-1}$ & $LIN(d_M,d_M)\rightarrow BC(d_{M-1})$ & $d_M^2$ \\
          \hline

    \multirow{2}[0]{*}{One below diagonal} & \cellcolor[rgb]{ .533,  .984,  .31} & 4    & $i=j+1$ & $i<M$ & $d_i \rightarrow  d_{i+1}  \times  d_{i}$ & $L_{\text{DS}}(1,1)\rightarrow BC(d_{i+1})$ & $2$ \\
          & \cellcolor[rgb]{ .973,  .588,  .549} & 5     &       & $i=M$ & $d_{M-1} \rightarrow  d_{M}  \times  d_{M-1}$ & $L_{\text{DS}}(1,d_M)$ & $2d_M$ \\
          \hline

    \multirow{2}[0]{*}{Lower triangular } & \cellcolor[rgb]{ .973,  .584,  .792} & 6*   & $j<i-1$ & $i<M$ & $d_j\rightarrow d_{i}  \times  d_{i-1}$ & $POOL(d_j)\rightarrow LIN(1,1)\rightarrow BC(d_{i-1}, d_i)$ & $1$ \\
          & \cellcolor[rgb]{ .153,  .463,  .725} & 7     &       & $i=M$ & $d_j\rightarrow d_{M}  \times  d_{M-1}$ & $POOL(d_j)\rightarrow LIN(1,d_M)\rightarrow BC(d_{M-1})$ & $d_M$ \\
          \hline

    \multirow{4}[0]{*}{Upper triangular} & \cellcolor[rgb]{ .992,  .941,  .333} & 8     & $j>i$ & $i=1 \text{ and } j<M$ & $d_j\rightarrow d_{1}  \times  d_0$ & $POOL(d_j)\rightarrow LIN(1,d_0)\rightarrow BC(d_1)$ & $d_0$ \\
          & \cellcolor[rgb]{ .588,  .051,  .322} & 9     &       & $i=1 \text{ and } j=M$ & $d_M\rightarrow d_{1}  \times  d_0$ & $LIN(d_M,d_0)\rightarrow BC(d_1)$ & $d_Md_0$ \\
          & \cellcolor[rgb]{ .122,  .424,  .396} & 10    &       & $i>1 \text{ and } j=M$ & $d_{M} \rightarrow  d_{i}  \times  d_{i-1}$ & $LIN(d_M,1)\rightarrow BC(d_{i-1},d_i)$ & $d_M$ \\
          & \cellcolor[rgb]{ .949,  .447,  0} & 11*  &       & $i>1\text{ and }  j<M$ & $d_{j} \rightarrow  d_{i}  \times  d_{i-1}$ & $POOL(d_j)\rightarrow LIN(1,1)\rightarrow BC(d_{i-1}, d_i)$ & $1$ \\
          \hline

    \end{tabular}%
    }
\label{tab:b2w}
\end{table*}%

\newpage
\begin{table*}[t]
  \centering
  \scriptsize
  \caption{Orbits for the action of $G$ on the indices of $\mathcal{V}$. These orbits define linear invariant layers and equivariant bias layers.}
  \vskip 0.11in
    \begin{tabular}{llll}
    \hline
    subspace & dimensionality & orbits & number of orbits \\
        \hline

    $W_1$ & $d_1 \times d_0$ & $O^{W_1}_j=\{ (i,j) \mid i\in[d_1]\}$ & $d_0$ \\
    $W_m,~ 1<m<M$ & $d_{m} \times d_{m-1}$ & $O^{W_m}_1=\{ (i,j)\mid i\in [d_{m}],j\in[d_{m-1}]\}$ & $1$ \\
    $W_M$ & $d_{M} \times d_{M-1}$ & $O^{W_M}_i=\{ (i,j) \mid j\in[d_{M-1}]\}$ & $d_M$ \\
    $b_i, 1<m<M$ & $d_m$   & $O^{b_m}_1=[d_i]$ & $1$ \\
    $b_M$ & $d_M$   & $O^{b_M}_i=\{i\}$ & $d_M$ \\
        \hline

    \end{tabular}%
  \label{tab:orbits}%
\end{table*}%
\section{Linear Maps Between Specific Weight and Bias Spaces}\label{appx:subrep_equi_maps}
\begin{proof}[Proof of Theorem \ref{thm:char_equi}]
As mentioned in the main text, by using proposition \ref{proposition:equi_direct_sum}, all we have to do in order to find a basis for the space of $G$-equivariant maps $L:\mathcal{V}\rightarrow \mathcal{V}$ is to find bases for linear $G$-equivariant maps between specific weight and bias spaces $\{\mathcal{B}_m,\mathcal{W}_\ell \}_{m,\ell\in[M]}$. To that end, we first use the rules specified in Section \ref{sec:charcterize} to create a list of layer types and their implementation. Then, one has to show that the layers in Tables \ref{tab:w2w}-\ref{tab:b2b} are linear, $G$-equivariant and that their parameters are linearly independent. This is straightforward. For example, the mappings between subspaces (e.g., $\mathcal{W}_j\to \mathcal{B}_i$) are clearly equivariant, as the composition of $G$-equivariant maps is $G$-equivariant. Finally and most importantly, we show that the number of parameters in the layers matches the dimension of the space of $G$-equivariant maps between the sub-representations, which can be calculated using Lemma \ref{lemma:dim_equi}.
Following are some general comments before we go over all layer types:
\begin{itemize}
    \item We use the fact that $\frac{1}{|S_n|}\sum_{\sigma\in S_n} tr(P(\sigma))^k=bell(k)$ \cite{maron2018invariant} (for the case $k=1,2$) where for a permutation $\sigma\in S_d$, $P(\sigma)\in \mathbb{R}^{d\times d}$ is its permutation representation. $bell(k)$ is the number of possible partitions of a set with $k$ elements.
    \item An index on which $G$ acts by permutation is called a \emph{set} index (or dimension). Other indices are called \emph{free} indices.
    \item Calculations of the dimensions of the equivariant maps spaces are presented below for the most complex weight-to-weight case. We omit the other cases (e.g., weight-to-bias) since they are very similar and can be obtained using the same methodology. \item Generally, shared set dimensions add a multiplicative factor of $bell(2)=2$ to the dimension of the space of equivariant layers, and free dimensions add a multiplicative factor equal to their dimensionality. Unsahred set dimensions add a multiplicative factor of $bell(1)=1$ so they do not affect the dimension of the equivariant layer space.
    \item In all cases below, $G$ is defined as in Equation \ref{eq:symmetry}, but the representations $\rho,\rho'$ of the input and output spaces, respectively,  differ according to the involved sub-representations
\end{itemize}

\textbf{$G$-equivariant linear functions between weight matrices.} A map between one weight matrix to another weight matrix is of the form $L:\mathbb{R}^{d\times d'}\rightarrow \mathbb{R}^{s \times s'}$. We will split into cases that cover all types of maps as appears in Table \ref{tab:w2w}, and compute the dimensions of the spaces of the equivariant layer below:
\begin{enumerate}
    \item (Two shared set indices). In that case, the layer is 4-dimensional ($bell(2)^2=4$) as we use the linear layers and dimension counting from \cite{hartford2018deep}.
    
    \item (Two shared indices, one set and one free) Assume $s=d$ are set indices, $s'=d'$ are free indices $\rho(g)=\rho'(g)=I_{d'}\otimes P(\sigma)$ for $\sigma\in S_d$ and 
    $$\frac{1}{|G|}\sum_{g\in G} tr(\rho(g)) \cdot tr(\rho'(g))=\frac{1}{|G|}\sum_{g\in G} tr(P(\sigma ))^2d'^2=bell(2)d'^2=2d'^2 $$
    We note that the summation over $G$ includes groups in the direct product that are trivially represented. This extra summation cancels the corresponding terms in $\frac{1}{|G|}$
    
    \item (One shared index, two set indices mapped to one shared set index, and one unshared free index). 
    Assume $s=d$ are shared set indices, $d'$ is another set index and $s'$ is free. $\rho(g)=P(\sigma)\otimes P(\tau)$,   $\rho_2(g)=P(\sigma)\otimes I_{s'}$ for $\sigma\in S_d,\tau\in S_{d'}$.
    $$\frac{1}{|G|}\sum_{g\in G} tr(\rho(g)) \cdot tr(\rho'(g))=\frac{1}{|G|}\sum_{g\in G} tr(P(\sigma))^2tr(P(\tau))tr(I_{s'})=bell(2)s'=2s' $$

    \item (One shared set index and unshared free index mapped to one shared set index, and one unshared set index). 
    Assume $s=d$ are shared set indices, $s'$ is another set index and $d'$ is free. $\rho(g)=P(\sigma)\otimes I_{d'}$,   $\rho'(g)=P(\sigma)\otimes P(\tau)$ for $\sigma\in S_d,\tau\in S_{s'}$.
    $$\frac{1}{|G|}\sum_{g\in G} tr(\rho(g)) \cdot tr(\rho'(g))=\frac{1}{|G|}\sum_{g\in G} tr(P(\sigma))^2tr(P(\tau))tr(I_{d'})=bell(2)s'=2d' $$

    \item (One shared index, two set indices mapped to one shared set index, and one unshared set index)
    Assume $s=d$ are set indices, $d',s'$ are other set indices 
    $\rho(g)=P(\sigma)\otimes P(\tau)$  .  $\rho'(g)=P(\sigma)\otimes P(\pi)$ for $\sigma\in S_d,\tau\in S_{d'},\pi\in s_{s'}$.
    $$\frac{1}{|G|}\sum_{g\in G} tr(\rho(g)) \cdot tr(\rho'(g))=\frac{1}{|G|}\sum_{g\in G} tr(P(\sigma))^2tr(P(\tau))tr(P(\pi))=bell(2)bell(1)^2=2 $$
    
    \item (One shared set index and unshared free index mapped to one shared set index and one unshared free index)\footnote{special case for $M=2$, not shown in Table \ref{tab:w2w}.}.
    Assume $s=d$ are set indices, $d',s'$ are other free indices 
    $\rho(g)=P(\sigma)\otimes I_{d'}$.  $\rho'(g)=P(\sigma)\otimes I_{s'}$ for $\sigma\in S_d$.
    $$\frac{1}{|G|}\sum_{g\in G} tr(\rho(g)) \cdot tr(\rho'(g))=\frac{1}{|G|}\sum_{g\in G} tr(P(\sigma))^2tr(I_{s'})tr(I_{d'}))=bell(2)bell(1)^2=2d's'$$
    
    \item (No shared indices, two set indices to one set and one free).
    Assume $d, d',s$ are unshared set indices, and $s'$ is a free index. 
    $\rho(g)=P(\sigma)\otimes P(\tau)$.  $\rho'(g)=P(\pi)\otimes I_{s'}$ for $\sigma\in S_d,\tau\in S_{d'},\pi\in S_s$.
    $$\frac{1}{|G|}\sum_{g\in G} tr(\rho(g)) \cdot tr(\rho'(g))=\frac{1}{|G|}\sum_{g\in G} tr(P(\sigma))tr(P(\tau))tr(P(\pi))tr(I_{s'})=bell(1)^3s'=s' $$
    
    \item (No shared indices, one set and one free indices map to other set and free indices). Assume $d, s$ are unshared set indices, and $d',s'$ are free index. 
    $\rho(g)=P(\sigma)\otimes I_{d'}$.  $\rho'(g)=P(\tau)\otimes I_{s'}$ for $\sigma\in S_d,\tau\in S_{d'}$.
        $$\frac{1}{|G|}\sum_{g\in G} tr(\rho(g)) \cdot tr(\rho'(g))=\frac{1}{|G|}\sum_{g\in G} tr(P(\sigma))tr(P(\tau))tr(P(I_{d'}))tr(I_{s'})=bell(1)^2d's'=d's' $$

    \item (No shared indices, one set one free map to two set indices). The calculation is the same as (7).
    
    \item (No shared indices, two sets map to two sets).  
    Assume $d,d',s,s'$ are set indices 
    $\rho(g)=P(\sigma)\otimes P(\tau)$.  $\rho'(g)=P(\omega)\otimes P(\pi)$ for $\sigma\in S_d,\tau\in S_{d'},\pi\in S_{s}, \omega \in S_{s'}$.
    $$\frac{1}{|G|}\sum_{g\in G} tr(\rho(g)) \cdot tr(\rho'(g))=\frac{1}{|G|}\sum_{g\in G} tr(P(\sigma))tr(P(\tau))tr(P(\pi))tr(P(\omega))=bell(1)^4=1 $$
\end{enumerate}

\end{proof}

\section{More Proofs for \cref{sec:charcterize}}\label{appx:proofs}

\begin{proof}[Proof of Proposition \ref{proposition:equi_direct_sum}]
The elements in $B$ are clearly linear as they are represented as matrices. It is also clear that they are linearly independent: equating a linear sum of these basis elements to zero implies that each block is zero since there are no overlaps between blocks. To end the argument we use the assumption that $B_{m\ell}$ are bases. Equivariance is also straightforward:
take a vector $v=\oplus v_m, v_m\in \mathcal{V}_m$ and a zero padded element $L^P\in B^P_{k\ell}$ that corresponds to an element $L\in B_{k\ell}$, then  $L^P\rho(g)v$ a zero-padded version of $L\rho_k(g)v_{k}$. On the other hand $\rho'(g)L^Pv$ is a zero-padded version of $\rho'_{\ell}(G)Lv_k$ and we get equality from the assumption that $L$ is equivariant.  

We now turn to prove that $B$ is  a basis. We do that by showing that the number of elements $B$ is equal to the dimension of the space of linear maps between $(\mathcal{V},\rho)$ and $(\mathcal{V}',\rho')$. We start by calculating the size of $B$. Clearly, $|B|=\sum_{m\ell}|B_{m\ell}|=\sum_{m\ell}dim(E(\rho_m,\rho_\ell))$ where $E(\rho_m,\rho_m')$ is the space of linear equivariant maps from $\rho_m$ to $\rho_{m'}$. 
On the other hand, using Lemma \ref{lemma:dim_equi} we get:

\begin{align}
        dim(E(\rho,\rho'))  &= \frac{1}{|G|}\sum_{g\in G}tr(\rho(g))\cdot tr(\rho'(g)) \\
       &= \frac{1}{|G|}\sum_{g\in G} \left(\sum_m tr(\rho_m(g))\right) \cdot \left(\sum_{\ell} tr(\rho'_{\ell}(g))\right)\\
       &=  \frac{1}{|G|}\sum_{g\in G} \sum_{m\ell} tr(\rho_m(g))\cdot tr(\rho'_{\ell}(g))\\
       &=   \sum_{m\ell} \left( \frac{1}{|G|}\sum_{g\in G}tr(\rho_m(g))\cdot tr(\rho'_{\ell}(g)) \right)
       \\ &=\sum_{m\ell} dim(E(\rho_m,\rho'_{\ell}))
                                \end{align}
Where we used the fact that the trace of a direct sum representation is the sum of the traces of the constituent sub-representations, and Lemma \ref{lemma:dim_equi} again in the final transition.

$$ $$
\end{proof}

\begin{lemma}[Dimension of space of equivariant functions between representations]\label{lemma:dim_equi}
Let $G$ be a permutation group, and let $(\mathcal{V},\rho)$ and $(\mathcal{V}',\rho')$ be orthogonal representations of $G$, then the dimension of the space of equivariant maps from $(\mathcal{V},\rho)$ and $(\mathcal{V}',\rho')$ is $\frac{1}{|G|}\sum_{g\in G} tr(\rho(g)) \cdot tr(\rho'(g))$
    
\end{lemma}
\begin{proof}
We generalize similar propositions from \cite{maron2018invariant,maron2020learning}. Every equivariant map $L$ is in the null space of the following set of linear equations: $L\rho(g)=\rho'(g)L$. Since $\rho'(g)$ is orthogonal we can write $\rho'(g)^TL\rho(g)=L$
which in turn can be written as $\rho(g)\otimes\rho' (g) vec(L)=vec(L)$ for all $g\in G$. The last equations define the space of linear functions $L$ that are fixed by multiplication with $\rho(g)\otimes\rho'(g)$. A projection onto this space is given by $\pi = \frac{1}{|G|}\sum_{g\in G}\rho(g)\otimes\rho'(g)$, and its dimension is given by the trace of the projection, namely $tr(\pi)=\frac{1}{|G|}\sum_{g\in G} tr(\rho(g))\cdot tr(\rho'(g))$ using the multiplicative law of the trace operator and Kronecker products.  
\end{proof}

\section{Proofs of \cref{prop:expressive}} \label{appx:expressive}

\begin{proof}[Proof of Proposition~\ref{prop:expressive}] Given input $v\in \mathcal{V}$, representing the weights of an MLP $f$ with a fixed number of layers and feature dimensions, and $x\in \mathbb{R}^{d_0}$ which is an input to this MLP, we wish to design a network $F$, composed of our affine equivariant layers, such that $F([x,v])$ approximates $f(x;v)$ in uniform convergence sense ($\| \cdot\|_{\infty}$). We note that $F$ is $G$-invariant. We assume the input to our network is both $v$ and $x$ and that these inputs are in some compact domain. Furthermore, we assume that the non-linearity function, $\sigma$ is a ReLU function for both $f$ and $F$ for simplicity, although this is not needed in general. 

Throughout the proof we will use the following basic operations: (1) \textit{Identity transformation}: Directly supported by our framework since it can be implemented using pointwise operations supported by our networks. (2) \textit{Summation over $d_i$ dimensions}: Directly supported by our framework since it is a linear equivariant operation, (3) \textit{Broadcasting over $d_i$ dimensions}: Directly supported by our framework, (4) \textit{feature-wise Hadamard product}: this is not directly supported in our framework. However, since Hadamard product is a pointwise continuous operation, we can implement an approximating MLP (in uniform convergence sense) on a compact domain using the universal approximation theorem \cite{hornik1991approximation}, (4) \textit{Non-linearity}: From our assumption, we can directly simulate the input networks non-linearities (otherwise, given another non-polynomial continuous activation,  we can use the universal approximation theorem to uniformly  approximate it).

Let $\hat{f}$ denote our current approximation for $f(x;v)$. 
To form the input to the equivariant network, we concatenate a broadcasted version of $x$, $X\in \mathbb{R}^{d_1\times d_0}$ to $W_1$ to form a tensor in $\mathbb{R}^{d_0\times d_1 \times 2}$. 
Our plan is to define a sequence of equivariant layers that will mimic a propagation of $x$ through the MLP $f(\cdot; v)$. Our current approximation $\hat{f}$ will be stored in an extra channel dimension.

We first wish to simulate $W_1 \odot X$ where $\odot$ denotes the Hadamard product. Let $f_{m_1}$ denote a $K_1$-layers MLP which approximate $f_{m_1}(x, y)\approx x \cdot y$ sufficiently well. We use $K_1$ consecutive mapping $\mathcal{W}_1\to \mathcal{W}_1$ to simulate $f_{m_1}$. Concretely, the mapping $\mathcal{W}_1\to \mathcal{W}_1$ is a DeepSets layer, $L(Z)_i=L_1(z_i)+L_2(\sum_{j\neq i} z_j)$. We set $L_1$ to the corresponding linear transformation from $f_{m_1}$ and $L_2=0$. We now have $\hat{f}\approx W_1\odot X$ at the location corresponding to $\mathcal{W}_1$. Next we use the layer $\mathcal{W}_1\to \mathcal{B}_1$  perform summation over the $d_0$ dimension to obtain $\hat{f}\approx W_1x$ as a second feature channel at location $\mathcal{B}_1$. Note that $\mathcal{W}_1\to \mathcal{B}_1$ is again a DeepSets layer that supports summation. We now have $[b_1, \hat{f}]\in \mathbb{R}^{2\times d_1}$ at location $\mathcal{B}_1$. Next we use the DeepSets mapping $\mathcal{B}_1\to\mathcal{B}_1$ to perform summation over the feature dimension to obtain $\hat{f}\approx W_1x+b_1$. Finally we apply non-linearity using the activation function of the equivariant network to obtain $\hat{f}\approx \sigma(W_1x+b_1)$.

We proceed in a similar manner. First broadcast $\hat{f}$ to a second feature dimension at location $\mathcal{W}_2$ using the DeepSets + Broadcasting layer $\mathcal{B}_1\to\mathcal{W}_2$. The mapping $\mathcal{W}_2\to\mathcal{W}_2$ is a $L_{\text{Har}}$ layer, so we can use a similar approach for simulating an MLP to approximate $\hat{f}\approx W_2\sigma(W_1x+b)$. Following the same procedure described above we can simulate the first $M-1$ layers of $F$, obtaining $\hat{f}\approx x_{M-1}$ at the position corresponds to $\mathcal{B}_{M-1}$.

Next, we use the DeepSets mapping $\mathcal{B}_{M-1}\to \mathcal{W}_{M}$ to broadcast $\hat{f}$ to a second feature dimension where $W_M$ is in the first feature dimension. Since $\mathcal{W}_{M}\to \mathcal{W}_{M}$ is a DeepSets layer, we can simulate the Hadamard product of the two feature dimensions to obtain $\hat{f}\approx W_M\odot x_{M-1}$. Next we use $\mathcal{W}_{M}\to \mathcal{B}_{M}$  to perform summation over $d_{M-1}$ and map it to a second feature dimension at location $\mathcal{B}_M$, with $B_M$ at the first dimension. Finally, we use the linear mapping $\mathcal{B}_{M}\to \mathcal{B}_{M}$ to sum the two feature dimensions to obtain $\hat{f}\approx f(x;v)=x_M$. 

We have constructed a sequence of DWSNets-layers that mimics a feed-forward procedure of an input $x$ on an MLP defined by a weight vector $v$. Importantly, all the operations we used are directly supported by our architecture, with the exception of the Hadamard product which was replaced by an approximation using the universal approximation theorem. To end the proof, we point out that uniform approximation is preserved by the composition of continuous functions on compact domains (see Lemma 6 in \cite{lim2022sign} for a proof).
\end{proof}

\section{Proof of \cref{prop:approx}} \label{appx:approx}

Our assumptions are listed below:
\begin{enumerate}
    \item  The inputs to the MLPs are in a compact domain $C_1\subset \mathbb{R}^{d_0}$
    \item Weights $v$ are in a compact domain $C_2\subset\mathcal{V}$
    \item We can map each weight $v$ to the function (MLPs) space $\mathcal{F_V}$ of functions represented by the weight, $f_v=f(\cdot;v)$. All such functions are $L_1$-Lipschitz w.r.t $||\cdot||_\infty$. 
    \item Given $x^{(1)},...,x^{(N)}\in C_1$ we define $vol_\sigma(y^{(1)},...,y^{(N)})=vol(\{v\in C_2:||(f(x^{(1)};v),...,f(x^{(N)};v))-(y^{(1)},...,y^{(N)})||_\infty\leq \sigma\})$. We assume that for all $\sigma>0$ (i)  $vol_\sigma(y^{(1)},...,y^{(N)})$ is continuous in both $y^{(1)},...,y^{(N)}$ and in $\sigma$ (ii) There exists $\lambda>0$ that for all $(y^{(1)},...,y^{(N)})$ in the range of $f(\cdot;v)$ for some $v\in C_2$ we have $vol_\sigma(y^{(1)},...,y^{(N)})>\lambda>0$. 
    \item Given a function $g$ defined on functions represented by our weights $g:\mathcal{F_V}\rightarrow \mathbb{R}$, we assume it is $L_2$-Lipschitz w.r.t the $||\cdot||_\infty$ norm on the function space.
\end{enumerate}

Below is a formal statement of \cref{prop:approx}.

\begin{proposition}\label{trm:universality}
     Let $g:\mathcal{F_V}\rightarrow \mathbb{R}$ be a $L_2$-Lipschitz function defined on the space of functions represented by M-layer MLPs with dimensions $d_0,...,d_M$, domain $C_1$,  ReLU nonlinearity, and weights in $v\in C_2\subset\mathcal{V}$. Assuming that all of the previous assumptions hold, then there exists a DWSNet with ReLU nonlinearities $F$  that approximates it up to $\epsilon$ accuracy, i.e. $\max_{v\in\mathcal{C}_2}|g(f_v)-F(v)|\leq \epsilon$. 
\end{proposition}

We split the proof into two parts, each stated and proved as a separate lemma.

\begin{lemma}\label{trm:simulation}
Let $M,d_0,\dots,d_M$ specify an MLP architecture. Let $C_1 \subset  \mathbb{R}^{d_0}$, $C_2 \subset \mathcal{V}$ be compact sets. For any $x^{(1)},...,x^{(N)} \in C_1$ there exists a DWSNet  $F$  with ReLU nonlinearities  that for any  $v\in C_2 $ outputs $F(v)$ with the following property $||F(v)-(f(x^{(1)};v),...,f(x^{(N)};v))||_\infty\leq\epsilon$.
\end{lemma}

\begin{proof} Given input $v\in C_2$, representing the weights of an MLP $f$ with a fixed number of layers and feature dimensions, and $x^{(1)},...,x^{(N)}\in C_1$ which are fixed inputs to the MLP, we wish to design a DWSNet $F$, composed of our linear equivariant layers, such that $F(v)$ approximates $(f(x^{(1)};v),...,f(x^{(N)};v))$. 
We do that in two steps. First, using the bias terms, our first layer concatenates a broadcasted version of $(x^{(1)},...,x^{(N)}),\,X\in\mathbb{R}^{d_1\times d_0\times N}$ to $W_1$ to form a tensor in $\mathbb{R}^{d_1\times d_0 \times (N+1)}$.  
We then use a similar construction to the one in the proof of \cref{prop:expressive} to find a network that approximates $f(x^{(i)}),~i=1\dots N$ in parallel.

\end{proof}

\begin{lemma}\label{trm:sim_to_target} Under all previously stated assumptions, if $x^{(1)},...,x^{(N)}$ are an $\epsilon$-net on the input domain $C_1$ w.r.t the infinity norm, i.e. $\max_{x\in C_1}\min_{i\in[N]}||x-x^{(i)}||\leq \epsilon$ then there exists an MLP $h$ such that for all weights $v\in C_2$ we have: 
$$||h(f_v(x_1),...,f_v(x_N))-g(f_v)||\leq 4L_1L_2\epsilon$$
\end{lemma}

\begin{proof}



We define $R(C_2)=\{(y^{(1)},...,y^{(N)}):\exists v\in C_2\,s.t.\,\forall\,i:\,f(x^{(i)};v)=y^{(i)}\}$. As the values $(f(x^{(1)};v),...,f(x^{(N)};v))$ are not enough to uniquely define $g(f_v)$ we will define a continuous approximation using smoothing. We define for all $(y^{(1)},...,y^{(N)})\in R(C_2)$ the function $\bar{g}_\sigma(y^{(1)},...,y^{(N)})$ as the average of of $g(f_v)$ over  $$A_\sigma(y^{(1)},...,y^{(N)})=\{v\in C_2:||(y^{(1)},...,y^{(N)})-(f(x^{(1)};v),...,f(x^{(N)};v))||_\infty\leq \sigma\},$$ 

 
i.e., $$\bar{g}_\sigma(y^{(1)},...,y^{(N)}) = \frac{1}{vol_\sigma(y^{(1)},...,y^{(N)})}\int_{A_\sigma(y^{(1)},...,y^{(N)})}g(f_v)dv.$$
We claim that $\bar{g}_\sigma$ is a continuous function of $y^{(1)},...,y^{(N)}$  due to the smoothing done:
Define $\bold{y}=(y^{(1)},...,y^{(N)})$ and $\tilde{\bold{y}}=(\tilde{y}^{(1)},...,\tilde{y}^{(N)})$ with $||\bold{y}-\tilde{\bold{y}}||\leq \delta$. We can look at $\bar{g}_\sigma(\bold{y})-\bar{g}_\sigma(\tilde{\bold{y}})$ as two integrals, one over $A_\sigma(\bold{y})\cap A_\sigma(\tilde{\bold{y}})$ and one over $A_\sigma(\bold{y})\triangle A_\sigma(\tilde{\bold{y}})$. The first part is equal to $\int_{A_\sigma(\bold{y})\cap A_\sigma(\tilde{\bold{y}})}g(f_v)dv\left(\frac{vol_\sigma(\tilde{\bold{y}})-vol_\sigma({\bold{y}})}{vol_\sigma(\bold{y})vol_\sigma(\tilde{\bold{y}})}\right)$ which goes to zero as $\delta$ goes to zero as the volume is continuous, bounded away from zero, and the integrand $g$ is also bounded. The integral on the symmetric difference is bounded by $C\cdot(vol_{\sigma+\delta}(\bold{y})-vol_{\sigma}(\bold{{y}}))+C\cdot(vol_{\sigma+\delta}(\tilde{\bold{y}})-vol_{\sigma}(\tilde{\bold{y}}))$ where $C$ is a bound on the integrand. This is because each point in the symmetric difference needs to be more than $\sigma$ away from $\tilde{\bold{y}}$ or ${\bold{y}}$, but no more than $\sigma+\delta$ away. This also goes to zero as the volume is continuous in $\sigma$ proving that $\bar{g}_\sigma$ is continuous. \\
 

Now that we showed that $\bar{g}_\sigma$ is continuous, we will show it is a good approximation to $g$. If $v\in \mathcal{V}$ and $f(x^{(i)};v)=y^{(i)}$ then $\bar{g}_\sigma (y^{(1)},\dots,y^{(n)})$  is an average of $g(f_{v'})$ over a set of $v'$ that differ on the $\epsilon$-net by at most $\sigma$. Let $x\in C_1$ and $x^{(i)}$ be its closest element of the net then 
\begin{align*}
&|f(x;{v})-f(x;{v'})|\leq   |f(x;{v})-f(x^{(i)};{v})| +|f(x^{(i)};{v})-f(x^{(i)};{v'})|+|f(x^{(i)};{v'})-f(x;{v'})|\leq 2L_1\epsilon+\sigma
\end{align*}
 i.e.,  $||f_{v}-f_{v'}||_\infty \leq 2L_1\epsilon+\sigma$. We can set $\sigma=L_1\epsilon$ and by the Lipschitz property of $g$ we get that the difference in the $g$ values of all averaged weights we average in $\bar{g}_\sigma$ is at most $3L_1L_2\epsilon$. This means that $||g(f_v)-\bar{g}_\sigma(y_0,...,y_N)||\leq 3L_1L_2\epsilon$. \\

Finally, we note that since $\bar{g}_\sigma$ is a continuous function over $R(C_2)$, which is compact as the image of a compact set by a continuous function, it can be approximated by an MLP $h$ such that $|h(y^{(1)},...,y^{(N)})-\bar{g}_\sigma(y^{(1) },...,y^{(N)})|\leq L_1L_2\epsilon$ to conclude the proof.
\end{proof}

\begin{proof}[Proof of \cref{trm:universality}]
From compactness of $C_1$ we have a finite $\epsilon_1$-net $x^{(1)},...,x^{(N)}$. 
From \cref{trm:simulation} there exists an invariant DWSNet $\bar{F}$ such that $\bar{F}(v)=(y^{(1)},...,y^{(N)})$ such that $||y^{(i)}-f(x^{(i)};v)||\leq \epsilon_2$. From  \cref{trm:sim_to_target} there exists an MLP $h$ such that $||h(f(x^{(1)};v),...,f(x^{(N)};v))-g(f_v)||\leq 4L_1L_2\epsilon_1$. We note that as $h$ is a ReLU MLP on a compact domain it is $L_3$-Lipshitz for some constant $L_3$. Now ${F}=h\circ \bar{F}$  is a DWSNet that has 
\begin{align*}
    &|{F}(v)-g(f_v)|=|h(\bar{F}(v))-g(f_v)|\leq 
    |h(\bar{F}(v))-h(f(x^{(1)};v),...,f(x^{(N)};v))|\\
    &+|h(f(x^{(1)};v),...,f(x^{(N)};v))-g(f_v)|\leq L_3\epsilon_2+4L_1L_2\epsilon_1
\end{align*}
The first we bound by the Lipschitz property of $h$ and the fact that $\bar{F}$ is an $\epsilon_2$ approximation and the second from lemma \ref{trm:sim_to_target}. We note that while $h$ depends on $x^{(1)},...,x^{(N)}$, and as such so does $L_3$, it does not depend on $\bar{F}$ or $\epsilon_2$, so we are free to pick $\epsilon_2$ based on the value of $L_3$ which concludes the proof.
\end{proof}

\section{Alternative Characterization Strategies} \label{appx:alternative}
We chose to work directly with the direct sum of the weight and bias spaces since it allows us to easily derive simple implementations for the layers. It should be noted that other strategies can be employed to characterize spaces of  linear equivariant layers, including decomposing $\mathcal{V}$ into \emph{irreducible} representations (as done on many previous works, e.g., \citet{cohen2016steerable,thomas2018tensor}). An advantage of this strategy is that it simplifies the structure of the blocks discussed above:  one can use a classic result called Schur's Lemma \cite{fulton2013representation}, which states that linear equivariant maps between irreducible representations are either zero or a scaled identity map. On the other hand, one might need to translate such a characterization back to the original weight and bias decomposition in order to implement the maps.

\section{Computational and Memory Requirements}

\revision{Like other equivariant architectures, such as CNNs and DeepSets \cite{zaheer2017deep}, DWS-layers have fewer parameters and are more computationally efficient than fully connected layers. In this section, we provide a brief comparison of DWS layers with fully connected layers in terms of parameter space and the time complexity of a feedforward pass.}

\revision{To simplify our analysis, we do not consider parallel computations in our comparison. Let $M$ denote the number of layers in the input MLP as $M$, and assume for simplicity that all feature dimensions (input, hidden, and output dimensions) of the MLP are equal, i.e., $d_i=d$. Under this setup, a DWS-layer requires $\mathcal{O}((M+d)^2)$ parameters. This is because all the inner blocks of a DWS-layer have a constant number of parameters (see Tables \ref{tab:w2w}-\ref{tab:b2w}). In contrast, an FC layer requires $\mathcal{O}((Md^2)^2)$ parameters.}

\revision{The time complexity of a feedforward pass is estimated under the assumption that DWS and FC layers are computed by independently performing the block operations and then aggregating them. It is worth noting that within each block, DWS-layers implement a specific parameter-sharing scheme that can be implemented efficiently. For example, one of the basic building blocks of a DWS-layer is the layer suggested by~\citet{hartford2018deep}, which parameterizes maps from an $d\times d$ dimensional matrix into another $d\times d$ dimensional matrix. In this case, an FC layer would require $\mathcal{O}(d^4)$ operations, whereas the Hartford layer~\cite{hartford2018deep} can be implemented efficiently using broadcast and pooling operations that require only $\mathcal{O}(d^2)$ operations. This argument can be extended to other block types as well. As a result, DWS-layers have asymptotically lower time complexity than fully connected layers.}

\section{Experimental and Technical Details}\label{app:exp_details}

\textbf{Data preparation.} In order to test our architecture on diverse data obtained from multiple independent sources, we train all input networks independently starting from different random seed (initialization). 
As preprocessing step, the networks are normalized as follows: let $v_i$ denote the $i$th weight vector in our dataset and let $v_{ij}$ the $j$th entry in $v_i$. Let $m(v)$ denote the average vector over the dataset and $s(v)$ the vector of standard deviations. We normalize each entry as $v_{ij} \leftarrow (v_{ij} - m(v)_{ij}) / s(v)_{ij}$. We empirically found this normalization to be beneficial and aid training. We split each dataset into three data splits, namely train, test and validation sets.

\textbf{Datasets.} We provide details for the network datasets used in this work. For INRs, we use the SIREN~\cite{sitzmann2020implicit} architecture, i.e., MLP with sine activation, otherwise, we use ReLU activation~\cite{agarap2018deep}.

\textit{Sine waves INRs for regression.} We generate dataset of $1000$ INRs with three layers and 32 hidden features, i.e., $1\to32\to 32\to1$. The input to the INR is a grid of size $2000$ in $[-\pi, \pi]$. We train the INRs using the Adam optimizer for $1000$ steps with learning-rate $1e-4$.
We use $800$ INRs for training and $100, 100$ INRs for testing and validation.

\textit{MNIST and Fashion-MNIST INRs.} We fit an INR to each image in the original dataset. We split the INR dataset into train, validation and test sets of sizes $55$K, $5$K, $10$K respectively. We train the INRs using the Adam optimizer for $1K$ steps with learning-rate $5e-4$. When the PSNR of the reconstructed image from the learned INR is greater than $40$, we use early stopping to reduce the generation time. Each INR consists of three layers with $32$ hidden features, i.e., $2\to 32 \to 32 \to 1$.

\textit{CIFAR10 image classifiers.} The data consists of $5000$ image classifiers. We use $4000$ networks for training and the remaining divided evenly between validation and testing sets. Each classifier consists of $5$ layers with $64$ hidden features, i.e., $3\cdot32^2=3072\to 64 \to 64 \to 64 \to 64 \to 10$. To increase the diversity of the input classifiers, we train each classifier on the binary classification task of distinguishing between two randomly sampled classes. We fit the classifiers using the Adam optimizer for $2$ epochs with learning-rate $5e-3$ and batch-size $128$.

\textit{Fashion-MNIST image classifiers.} We fit $200$ image classifiers to the Fashion-MNIST dataset with $10$ classes. Each classifier consists of $4$ layers with $128$ hidden features, i.e., $28^2=784\to 128 \to 128 \to 128 \to 10$. We fit the classifiers using the Adam optimizer for $5$ epochs with learning-rate $5e-3$ and batch-size $1024$. To generate classifiers with diverse generalization performance, we save a checkpoint of the classifier's weights along with its generalization performance every $2$ steps throughout the optimization process. We use $150$ optimization trajectories for training, and the rest are divided evenly between validation and testing sets.

\textit{Sine waves INRs for SSL.} We fit $5000$ INRs to sine waves of the form $a\sin(bx)$ on $[-\pi, \pi]$. Here $a,b\sim U(0, 10)$ and $x$ is a grid of size $2000$. We use $4000$ samples for training and the remaining INRs divided evenly between validation and testing sets. We train the INRs using the Adam optimizer with a learning-rate of $1e-3$ for $1500$ steps. Each INR consists of three layers and 32 hidden features, i.e., $1\to32\to32\to1$. 

\textbf{Hyperparameter optimization and early stopping.} For each learning setup and each method we search over the learning rate in $\{5e-3, 1e-3, 5e-4, 1e-4 \}$. We select the best learning rate using the validation set. Additionally, we utilize the validation set for early stopping, i.e., select the best model w.r.t. validation metric.

\textbf{Data augmentation.}\label{app:data_augmentation} We employ data augmentation in all experiments and for all methods. For non-INR input networks, we augment the weight vector with Gaussian and dropout noise. For INRs we can apply a wider range of augmentations: For example consider an INR for an image $f: \mathbb{R}^2\to \mathbb{R}^3$. Let $x$ denote a grid in $[0,1]^2$ which is the input to the INR. We can apply data augmentation to the weight vector to simulate augmentation on the image it represents. As a concrete example let $R\in \mathbb{R}^{2\times 2}$ denote a rotation matrix. By multiplying $W_1$ with $R$ we are rotating the image represented by the INR. Similarly, we can translate the image, or change its scale.
For INRs, we apply rotation, translation, and scaling augmentations, along with Gaussian and dropout noise. See Figure~\ref{fig:inr_aug} (right) for an example of data augmentations.

\begin{figure*}
\centering
\begin{subfigure}
  \centering
  \includegraphics[width=.3\textwidth]{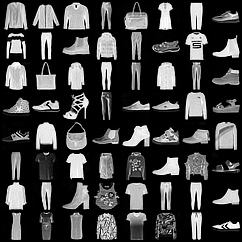}
  \label{fig:inr_aug_1}
\end{subfigure}%
\qquad
\begin{subfigure}
  \centering
  \includegraphics[width=.3\textwidth]{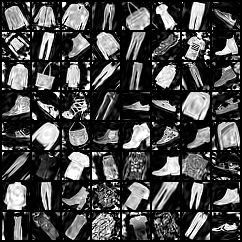}
  \label{fig:inr_aug_2}
\end{subfigure}
\caption{\textit{INR reconstruction for Fashion-MNIST}: INR reconstruction for Fashion-MNIST images (left) and INR reconstruction of the same images after (weight) data augmentation (right).}
\label{fig:inr_aug}
\end{figure*}

\textbf{Initialization.} We found that appropriate initialization is important when the output of the model is used to parametrize a network, e.g., in the domain adaptation experiment. Similar observations where made in the Hypernetwork literature~\cite{chang2019principled,litany2022federated}. We use a similar initialization to that used in~\citet{litany2022federated}. Specifically, for weight matrices we use Xavier-normal initialization multiplied $\mu\sqrt{2d_{in}/d_{out}}$. For invariant tasks we set $\mu=1$. For tasks where the output parameterize a network we set $\mu=1e-3$.

\textbf{Methods to control the complexity of \ourmethod{}.} Since our network complexity is controlled by $d_0$ and $d_M$, the number of parameters can grow large when the input and/or output dimensions are large. To control the number of parameters we first apply a linear transformation to the input/output dimension to map it to a lower dimension space, e.g., $LIN(d_0, d_0')$ with $d_0'\ll d_0$. We can then apply another linear transformation (if needed) $LIN(d_0',d_0)$ to map the output back to the original space.
We note that since $d_0,d_M$ are \emph{free} indices, it can be modified between layers while still maintaining $G$-equivariance.

\textbf{Approximation for model alignment.} In the literature several studies suggested methods to align the weights/neurons of NNs (e.g., \cite{ashmore2015method,singh2020model, ainsworth2022git}). Here, we chose the method presented in \cite{ainsworth2022git}. \citet{ainsworth2022git} suggested an iterative approach for aligning many models termed \textit{MergeMany}. The basic idea is to run the alignment algorithm at each iteration between one of the models and an average of all the other models. This algorithm is guaranteed to converge. However, the convergence time depends on both the number of models and the allowed alignment error. For instance, on the MNIST classification task, we waited more than $24$ hours before stopping the algorithm and it still didn't finish even one iteration with an error of $1e-12$ (the default error in the official GitHub repository). Therefore, we ran this method according to the following scheme; we first fixed the error to $1e-3$, then we randomly chose a sub-sample of $1000$ models and ran the \textit{MergeMany} algorithm on them only. The result from this process was a new model (the average of the aligned models) which we used for aligning the remaining training models and the test models to get a training set and test set of aligned models.

\textbf{Additional experimental details.} 
Unless stated otherwise, we use \ourmethod{} with $4$ hidden equivariant layers, and a final invariant layer, when appropriate. Additionally, we use max-pooling (as the $POOL$ components) in all experiments.
The baseline methods are constructed to match the depth of the \ourmethod{}, with feature dimensions chosen to match the capacity (number of parameters) of the \ourmethod{}, for a fair comparison.
We train all methods with ReLU activations and Batch-Normalization (BN) layers. We found BN layers to be beneficial in terms of generalization performance and smoother optimization process. We train all methods using the AdamW \cite{loshchilov2017decoupled} optimizer with a weight-decay of $5e-4$. We use the validation split to select the best learning rate in $\{5e-3, 1e-3, 5e-4, 1e-4\}$ for each method. Additionally, we use the validation split to select the best model (i.e., early stopping). We repeat all experiments using 3 random seeds and report the average performance along with the standard deviation for the relevant metric.

\begin{figure*}[t]
    \centering
    \includegraphics[width=1.\linewidth, clip]{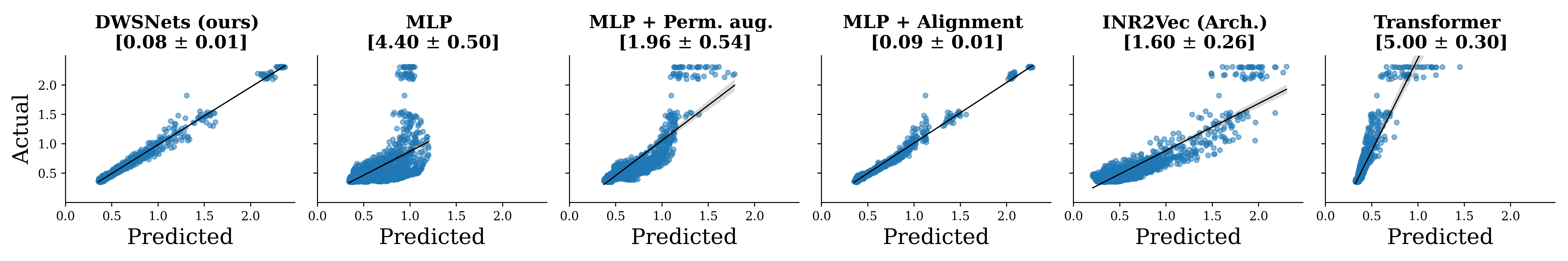}
    \caption{\textit{Predicting the generalization performance of NNs:} Given the weight vector $v$ the task is to predict the performance of $f(\cdot;v)$ on the test set. We report $100\times$MSE averaged over 3 random seeds. Black lines illustrate a linear fit to the predicted-vs-actual data points.}
    \label{fig:fmnist_gen}
\end{figure*}

\textit{Regression of sine waves.} We train a \ourmethod{} with two hidden layers and $8$ hidden features. All networks consist of $\sim 15$K parameters. We use a batch size of $32$ and train the models for $100$ epochs.

\textit{Classification on INRs.} We train all methods for $100$ epochs. All networks consist of $\sim 550$K parameters. We use a batch size of $512$.

\textit{Predicting the generalization error of neural networks.} We train all methods for $15$ epochs ($\sim 15$K steps). For \ourmethod{} we map $d_0=784$ to $d_0'=16$, and use a $4$-hidden layers network with $16$ features. All networks consist of $\sim 4$M parameters. 
We use a batch size of $32$.

\textit{Learning to adapt networks to new domains.} We train all methods for $10$K steps. For \ourmethod{} we map $d_0=3072$ to $d_0'=16$, and use a $4$-hidden layers network with $16$ features. All networks consist of $\sim 4$M parameters. At each training step, we sample a batch of input classifiers and a batch of images from the source domain. We map each weight vector $v$ to a residual weight vector $\Delta v$. We then pass the image batch through all networks parametrized by $v-\Delta v$ and update our model according to the obtained classification (cross-entropy) loss.
We use a batch size of $32$ for input networks and $128$ for images.

\textit{Self-supervised learning for dense representation.} We train the different methods for $500$ epochs with batch-size of $512$. The \ourmethod{} is consists of 4-hidden layers with $16$ features. We set the dense representation dimension to $16$. All networks consist of $\sim 100$K parameters. We use a temperature of $0.1$ to scale the NT-Xent loss~\cite{chen2020simple}.

\section{Additional Experiments}\label{sec:additional_exp}

\subsection{Predicting the Generalization Error of Neural Networks.} 

Given an MLP classifier, we train a DWSNet to predict its generalization performance, defined as the test error on a held-out set (see also \cite{schurholthyper}). To create a dataset for this problem, we train 200 MLP image classifiers on the Fashion-MNIST dataset. We save checkpoints with the classifier's weights throughout the optimization process, together with its generalization error.  
Then, we train a DWSNet to predict the generalization performance from the classifier's weights. Figure~\ref{fig:fmnist_gen} shows that DWSNet achieves the lowest error, significantly outperforming most baselines.

\subsection{Dense Representation} \label{app_sec:dense_rep}
Here we give the full results for learning a dense representation that was presented in Section \ref{sec:Experiments}. Figure \ref{fig:dense_all} shows that \ourmethod{} generates an embedding with a clear and intuitive 2D structure. That is, we can notice a representation that groups models with similar frequencies and amplitudes together and a gradual change between the different regions. On the other hand, other baselines don't seem to have this nice explainable property. 
\begin{figure}[!t]
    \centering
\includegraphics[width=.8\linewidth, clip]{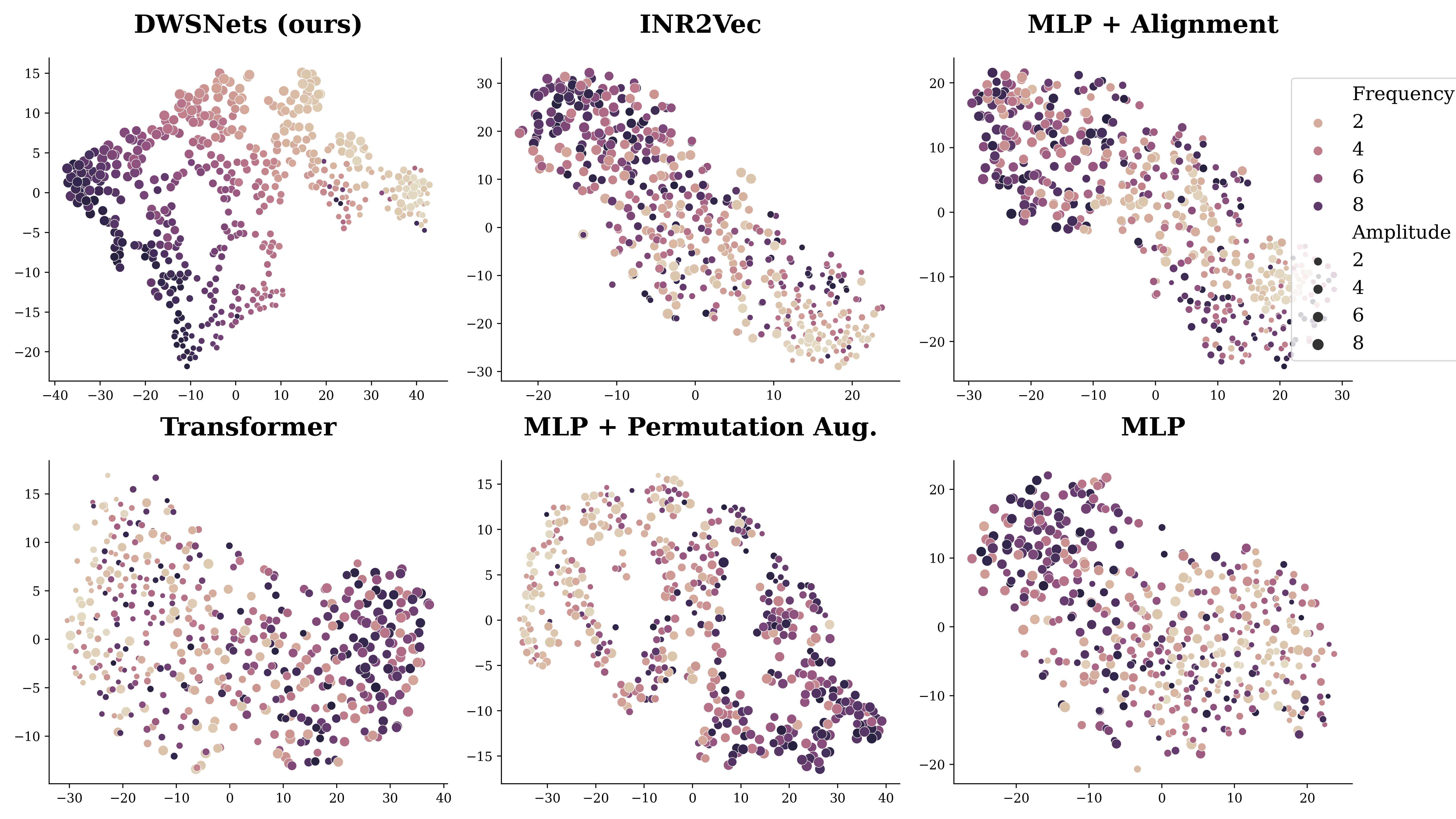}
    \caption{\textit{Dense representation:} 2D TSNE of the resulting low-dimensional space.}
    \label{fig:dense_all}
\end{figure}






\subsection{Ablation Study}

\begin{table}[!h]
\centering
\small
\caption{\textit{Ablation on the DWSNet's blocks using the MNIST INRs dataset.}}
\vskip 0.11in
\begin{tabular}{lcc}
\toprule
& Test Acc. & \# params \\
\midrule
B2B  & $65.87 \pm 0.37$ & $65$K \\
W2W & $84.75 \pm 1.11$ & $235$K\\
W2W + B2B & $85.23 \pm 0.01 $ & $300$K \\
Diagonal & $ 85.68\pm 0.42 $ & 460K \\
DWSNets & $85.71\pm 0.57$ & 550K \\
\bottomrule
\end{tabular}
\label{tab:ablation}
\end{table}

Here we investigate the effect of using only part of the blocks in our proposed architecture, using the classification task of MNIST INRs.
\revision{We compare the bias-to-bias (B2B), weight-to-weight (W2W) and a \emph{``diagonal"} version of DWSNets consists only of internal blocks which maps joint set dimensions. For example, for the B2B block, these internal blocks form the main diagonal.}
\revision{The results are presented} in Table~\ref{tab:ablation}. Not surprisingly, the W2W block is the most contributing factor to the overall performance as it conveys most of the information. Nevertheless, adding the other blocks increase the overall performance. \revision{Interestingly, the diagonal DWSNet achieves performance similar to those obtained with all blocks.}

\subsection{The importance of data augmentation and batch normalization}

\begin{table}[!h]
\centering
\small
\caption{\textit{The effect of BN and DA:} Performance improvement through Data Augmentation and Batch Normalization on the MNIST INRs classification task.}
\vskip 0.11in
\begin{tabular}{lc}
\toprule
& Test Acc. \\
\midrule
DWSNet  & $77.20 \pm 0.41$ \\
DWSNet + DA & $80.20 \pm 0.28$ \\
DWSNet + BN & $83.05 \pm 1.35 $ \\
DWSNet + DA + BN & $85.71\pm 0.57$\\
\bottomrule
\end{tabular}
\label{tab:bn_da}
\end{table}

\revision{Throughout our experiments, we have consistently found data augmentation (DA) and batch normalization (BN) to be highly beneficial techniques in improving model performance. In this subsection, we present the results obtained by applying these techniques to the MNIST INRs classification task. Here we apply the data augmentation techniques for INRs described in Appendix~\ref{app:data_augmentation}. Our findings highlight the importance of data augmentation and batch normalization in improving the performance of DWSNets. The results are presented in Table~\ref{tab:bn_da}.}

\subsection{Challenging cases}\label{app:failure-cases}

Here we discuss two challenging cases that we encountered while experimenting with our method. 

\textit{Learning to prune.} One possible application of \ourmethod{} is to learn how to prune a network. Namely, given an input network it learns to output a mask that dictates which parameters from the input network to drop and which ones to keep. To evaluate our method on this task we used INRs generated based on the div2k dataset \cite{Agustsson_2017_CVPR_Workshops}. The loss function was to reconstruct the original image while regularizing the mask to be as sparse as possible. We tried different techniques to learn such a mask inspired by common solutions in the literature (e.g., \cite{hubara2016binarized}). Unfortunately, \ourmethod{} showed a tendency to prune many parameters of the same layers while keeping other layers untouched. We believe that this issue can be solved by a proper initialization and we see this avenue as a promising research direction for leveraging \ourmethod{}.

\textit{Working with INRs.} In some cases, we found it challenging to process INRs. Consider the problem of classifying CIFAR10 INRs to the original ten classes. In our experiments, we found that while significantly outperforming baseline methods, \ourmethod{} achieve unsatisfactory results in this task. A possible reason for that is that the INR, as a function from $\mathbb{R}^2$ to $\mathbb{R}^3$ is only informative on $[0, 1]^2$. Hence, it is possible that when processing these functions (parameterized with the weight vectors), with no additional information on the input domain, the network relies on the underlying, implicit noise signal originating from outside the training domain, i.e., $\mathbb{R}\setminus [0,1]^2$. If that is indeed the case, one potential solution would be to encourage the INR's output to be constant on $\mathbb{R}\setminus [0,1]^2$.

\end{document}